%% file: ms.tex
\documentclass[preprint,12pt,sort&compress]{elsarticle}

\usepackage{amsmath}
\usepackage{amssymb}
\usepackage{amsthm}
\usepackage{mathrsfs}

\usepackage[compatibility=false]{caption}
\usepackage{subcaption}

\usepackage{enumitem}

\usepackage[retainorgcmds]{IEEEtrantools}

\usepackage{multirow}

\usepackage[usenames,dvipsnames]{xcolor}

\usepackage{tikz}

\usepackage{setspace}

\usepackage{pdfcomment}

\RequirePackage[linesnumbered,ruled,vlined]{algorithm2e}

\theoremstyle{definition}
\newtheorem{myDef}{Definition}

\theoremstyle{plain}
\newtheorem{myTheorem}{Theorem}

\theoremstyle{plain}
\newtheorem{myLemma}{Lemma}

\theoremstyle{plain}
\newtheorem{myCorollary}{Corollary}

\theoremstyle{definition}

\theoremstyle{remark}
\newtheorem{myExample}{Example}

\newenvironment{myExampleCont1}{
   \addtocounter{myExample}{-1} \begin{myExample}[continued]}{
   \end{myExample}
   }
  
\newenvironment{myExampleCont2}{
   \addtocounter{myExample}{-2} \begin{myExample}[continued]}{
   \end{myExample}
   }

\newcommand*\circled[1]{\tikz[baseline=(char.base)]{
  \node[shape=circle,draw,inner sep=0pt] (char) {#1};}}

\newcommand\independent{\protect\mathpalette{\protect\independenT}{\perp}}
\def\independenT#1#2{\mathrel{\rlap{$#1#2$}\mkern2mu{#1#2}}}

\journal{arXiv}

\begin{document}
\input{NT_DCEG_rev8}

\end{document}

%% file: NT_DCEG_rev8.tex
\begin{frontmatter}



\title{An $N$ Time-Slice Dynamic Chain Event Graph}

\author[label1]{Rodrigo A. Collazo \corref{cor1}}
\ead{Collazo@marinha.mil.br}

\author[label2]{Jim Q. Smith}
\ead{J.Q.Smith@warwick.ac.uk}

\address[label1]{Department of Systems Engineering, Naval Systems Analysis Centre, Rio de Janeiro 20091-000, Brazil }

\address[label2]{Department of Statistics, University of Warwick, Coventry CV4 7AL, United Kingdom }

\cortext[cor1]{Corresponding author}

\begin{abstract}
The Dynamic Chain Event Graph (DCEG) is able to depict many classes of discrete random processes exhibiting asymmetries in their developments and context-specific conditional probabilities structures.  However, paradoxically, this very generality has so far frustrated its wide application. So in this paper we develop an object-oriented method to fully analyse a particularly useful and feasibly implementable new subclass of these graphical models called the $N$~Time-Slice~DCEG (${N\text{T-DCEG}}$). After demonstrating a close relationship between an $N$T-DCEG and a specific class of Markov processes, we discuss how graphical modellers can exploit this connection to gain a deep understanding of their processes. We also show how to read from the topology of this graph context-specific independence statements that can then be checked by domain experts. Our methods are illustrated throughout using examples of dynamic multivariate processes describing inmate radicalisation in a prison.  
\end{abstract}


\end{frontmatter}


%%%%%%%%%%%%%%%%%%%%%%%%%%%%%%%%%%%%%%%%%%%%%%%%%%%%%%%%%%%%%%%
\section{Introduction}
\label{sec:Introduction}

Discrete multivariate dynamic models have been studied widely. A new model class for discrete longitudinal data, called Dynamic Chain Event Graph (DCEG) \cite{Barclay.etal.2015}, has now been added to these modelling tools. A DCEG is a natural counterpart of a Chain Event Graph (CEG) \cite{Smith.Anderson.2008,Collazo.etal.2018}, which has proved to be useful for modelling discrete processes on populations exhibiting context-specific conditional independences. In this paper, we further advance the foundation of the DCEG model class in discrete time. We focus on the following objectives:
\begin{enumerate}[nosep]
\item We revisit the definition of DCEG models to circumvent the topological ambiguities when using the definition given by~\cite{Barclay.etal.2015}. 
\item We introduce a new model class of practical DCEGs called $N$~Time-Slice Dynamic Chain Event Graph ($N\text{T-DCEG}$). 
\item We develop an effective methodology to guide the construction and use of $N\text{T-DCEGs}$ based on objects that are components networked together through an infinite tree.
\end{enumerate}

In~\cite{Barclay.etal.2015} we defined a very general dynamic class of DCEG models that directly extends the CEG finite tree-based semantics to infinite trees without formally defining the probabilistic infinite tree-based model. This definition was originally designed around a graphical representation for models corresponding to simple semi-Markov processes. We subsequently found that many dynamic frameworks found in practice -- like the ones we use in the illustrations -- are most compellingly described in terms of properties of a Markov and not a semi-Markov process. Then the graphical model we originally built was not adequate to fully and unambiguously represent and interrogate the sorts of hypotheses we might have about such dynamic processes. As showed in~\cite{Collazo.2017, Collazo.Smith.2017}, these inadequacies are often associated with some topological ambiguities that prevent us from developing a rigorous framework for modelling the target process.

Here we therefore reset the DCEG framework and then define a new class of graphical model -- the $N\text{T-DCEG}$ -- specifically for modelling time-homogeneous Markov processes. An~$N\text{T-DCEG}$ enables us to construct a large-scale model for the whole infinite process based on a finite set of objects representing finite sub-processes. It directly supports learning and reasoning, and can be interpreted causally~\cite{Collazo.2017,Collazo.Smith.2018}. 

In addition, it is shown in~\cite{Collazo.2017,Collazo.Smith.2018}
	 that the class of Dynamic Bayesian Networks (DBNs) \cite{Dean.Kanazawa.1989,Nicholson.1992,Kjaerulff.1992}, a provenly useful class of dynamic models \cite{Dabrowski.Villiers.2015,Rubio.etal.2014,Marini.etal.2015,Li.etal.2014,Wu.etal.2015,Sun.Sun.2015,Khakzad.2015}, is a simple special case of this new DCEG subclass. Recall that a DBN corresponds to an extension of a BN for modelling and reasoning within dynamic systems whose progress is recorded over a sequence of discrete time steps. One consequence of this property is that a DBN can be refined using this subclass of DCEGs containing better fitting but still interpretable models to explain a given data set. For an example of an analogous link between Bayesian Networks (BNs) and CEGs in a non-dynamic setting, see~\cite{Barclay.etal.2013a}. 

Despite its flexibility, because it is based on a Directed Acyclic Graph a DBN and a BN share some well-known limitations \cite{Poole.Zhang.2003}. For example, processes are always defined using a preassigned collection of random vectors. So in any context where it is artificial to define a process directly through a set of conditional probabilities between the given components of a multivariate time series then its representation using a DBN can be restrictive and is often not appropriate. This is especially true when a process is characterised by context-specific conditional independences and highly asymmetric developments described in terms of observed events rather than random variables. Even if we attempt to embellish the BN/DBN models with context-specific information \cite{Boutilier.etal.1996, Poole.Zhang.2003, McAllester.etal.2008} or objects \cite{Koller.Pfeffer.1997,Bangso.Wullemin.2000} these embellishments will remain hidden in the conditional probability tables and not expressed graphically. 

One alternative class of graphs to the BN/DBN which is able to at least depict structural asymmetries directly is an event tree~\cite{Shafer.1996}. Building on it a CEG/DCEG model is obtained in the following three major steps: the representation of the qualitative structure of the process using an event tree~$\mathcal{T}$, a finite one for a CEG and an infinite one for a DCEG; the embellishment of~$\mathcal{T}$ with a probability measure~$\mathcal{P}$ through colours to obtain the staged tree~$\mathcal{ST}$; and eventually the transformation of~$\mathcal{ST}$ into the CEG/DCEG graph according to some simple graphical rules that collapse every repetitive structure of the staged tree into a single vertex or edge.

A CEG/DCEG model provides a more compact representation of the process than its corresponding staged tree although both graphs encapsulate the same information represented by a pair $(\mathcal{T},\mathcal{P})$. The elimination of redundant structure is particularly important in the dynamic setting where the process is supported by an infinite event tree but its DCEG model may nevertheless be a finite direct cyclic graph under some usual assumptions that will be introduced here. This graphical synthesis contrasts to other tree-based models such as a classification tree~\cite{Breiman.etal.1984} and a decision tree~\cite{Quinlan.Rivest.1989, Friedman.Goldszmidt.1998, Koller.Friedman.2009}.

We have noted that decision trees \cite{Friedman.Goldszmidt.1998} were used to explore context-specific local structures in BNs. A decision tree is constructed for each variable with which there are context-specific conditional independences associated in a BN. Thus, if a BN has context-specific independences corresponding to two or more variables, then it will be necessary to draw multiple decision trees to depict them. In contrast to a CEG/DCEG depicts all context-specific conditional independences using a single tree-based graph that represents the whole process. By not demanding to define a process using a pre-defined set of random variables, a CEG/DCEG captures broader local structures and provides us with a more flexible frame to concatenate the context-specific statements. It also provides graphical support for embodying logical constraints and managing sparse conditional probability tables without requiring additional dummy or degenerate variables. Example~\ref{ex:Radicalisation_Static_3variables} illustrates the advantages of this three-step framework using a CEG to model a very simple non-dynamic process of radicalisation of inmates.

\begin{myExample}
\label{ex:Radicalisation_Static_3variables}

\emph{
	The process by which a male prisoner might be radicalised is often explained through his social networks and how the population develops in time~\cite{Hannah.etal.2008,Neumann.2010,Silke.2011}. Although this example could be elaborated into a much larger and more realistic class, for the purpose of illustrating some methodological concepts in this paper we consider a simplification of this highly challenging process. We start constructing a vanilla CEG model for a non-dynamic setting that enables us to incorporate some dynamic structures in the subsequent sections.}

\emph{
	The physical movements and social interactions of prisoners are constantly being monitored and recorded in prisons. This helps experts to measure the frequency that a ``standard" prisoner  is able to socially interact with other prisoners who are identified as potential recruiters to radicalisation. Here this measure can take one of the following three levels: $\text{$s$- sporadic}$, $f$- frequent, $i$- intense. Experts also need to classify a prisoner into one of the following three categories with respect to his degree of radicalisation: resilient to ($r$), vulnerable to ($v$) or adopting ($a$) radicalisation. In each cycle of observations experts usually perform this classification based on two social-psychological drivers named social alienation and motivation to violence.} 
	
\emph{
	In a non-dynamic model, it is also important to understand how social interactions and the degree of radicalisation impact the probability of an inmate to remain in the prison (event~$n$) or to be transferred to another prison (event~$t$). Our focus will not be on developing methods for tracking a given prisoner or examining policy impacts across many prisons. Instead our single objective will be to help explain and predict how the population of a given prison might become radicalised.}

\emph{
	To represent this process using a CEG we first need to construct its corresponding event tree (Figure~\ref{fig:FiniteEventTree_Radicalisation}) that provides a snapshot of the multiple ways that the different events may unfold. The tree supports a modelling representation that allows domain experts to describe the qualitative aspects of the process using the observed events rather than random variables.  
}

\vspace{0pt}
\begin{figure}[t]
	\begin{center}
		\includegraphics[scale=0.35,angle=-90,origin=c,trim=0 135 0 0] {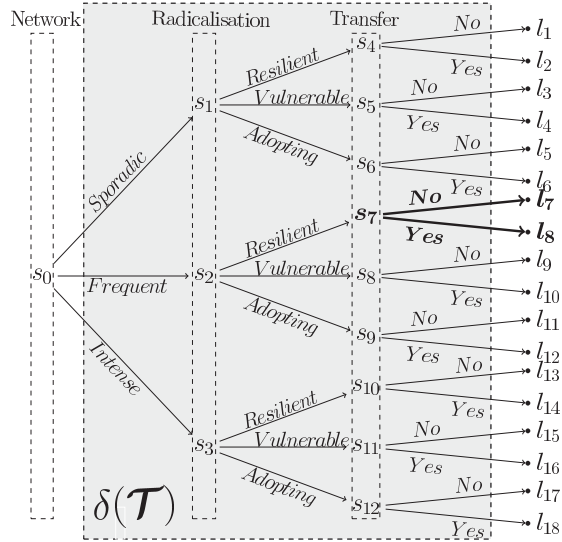}
	\end{center}
	\vspace{15pt}
	\caption{Event Tree~$\mathcal{T}$ \label{fig:FiniteEventTree_Radicalisation}}
	\vspace{0pt}
\end{figure} 

\emph{
	Experts can now use the event tree to elicit the conditional independence structures that may be embedded into the process. For instance, assume that a prison manager tells us that it is only practical to routinely transfer a very small number of inmates at any given time since he has to follow a rigorous and time-consuming bureaucratic procedure. He explains that in this context his focus tends to be on inmates who are suspect to be recently radicalised. For this reason, he believes that the probability to transfer an inmate given his degree of radicalisation does not depend upon his level of socialisation with well-known recruiters. Moreover, he also thinks that all those prisoners who have not adopted radicalisation are equally likely to be transferred. Figure~\ref{fig:FiniteStagedTree_Radicalisation} depicts the corresponding staged tree of this process after embellishing the event tree with a probability measure depicted by colours. To obtain the related CEG (Figure~\ref{fig:CEG_Radicalisation}), a modeller only needs to merge all vertices coloured the same into a single vertex and gathering all vertex~$l_i$, $i=1,\ldots,18$, into a single vertex~$w_\infty$. } 

\emph{
	Even without a deep mathematical understanding of the semantics of a staged tree or a CEG, we can intuitively read from these graphs that the conditional probability of transfer is identical given that an inmate is resilient or vulnerable. Observe that a CEG wraps all similar information in a common graphical structure, providing domain experts with a dense instructive summary of the whole process. Despite not being the case in this example, if an event has probability zero, then this can be expressed graphically by simply omitting its corresponding edge in the CEG. Such structural features play an important role in common dynamic processes whose probability conditional tables tend to contain many structural zeros.   
} 

%\vspace{39pt}
\begin{figure}[t]
	\begin{center}
		\begin{subfigure}{.5\textwidth}
			\begin{flushleft}
				$\;\,$\includegraphics[scale=0.31,angle=-90,origin=c,trim=130 75 0 0] {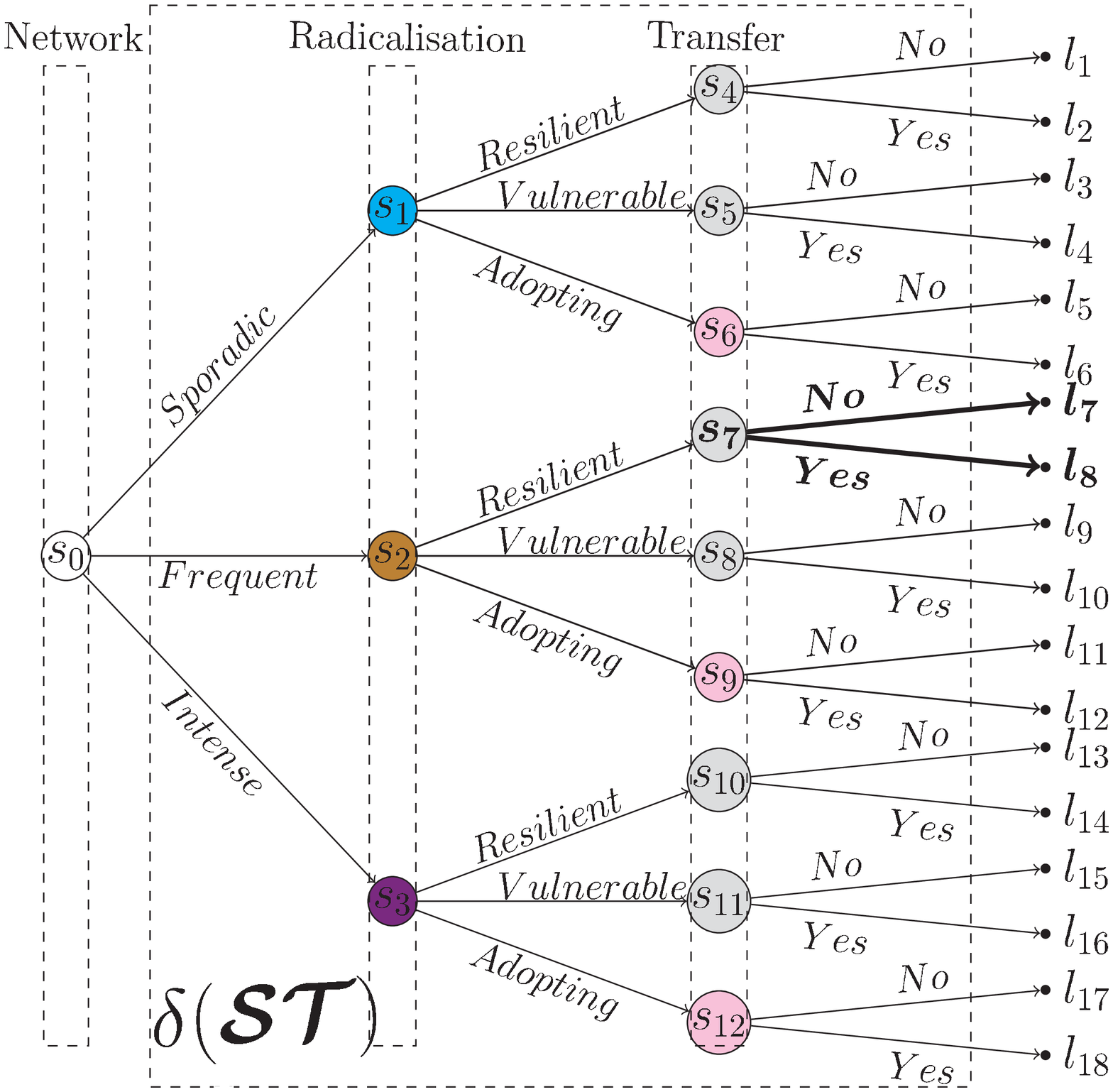}
			\end{flushleft}
			\vspace{15pt}
			\caption{Staged Tree~$\mathcal{T}$ \label{fig:FiniteStagedTree_Radicalisation}}
		\end{subfigure}%
		\begin{subfigure}{.395\textwidth}
			\begin{flushright}
				\includegraphics[scale=0.215,angle=-90,origin=c,trim=100 10 0 -90]
				{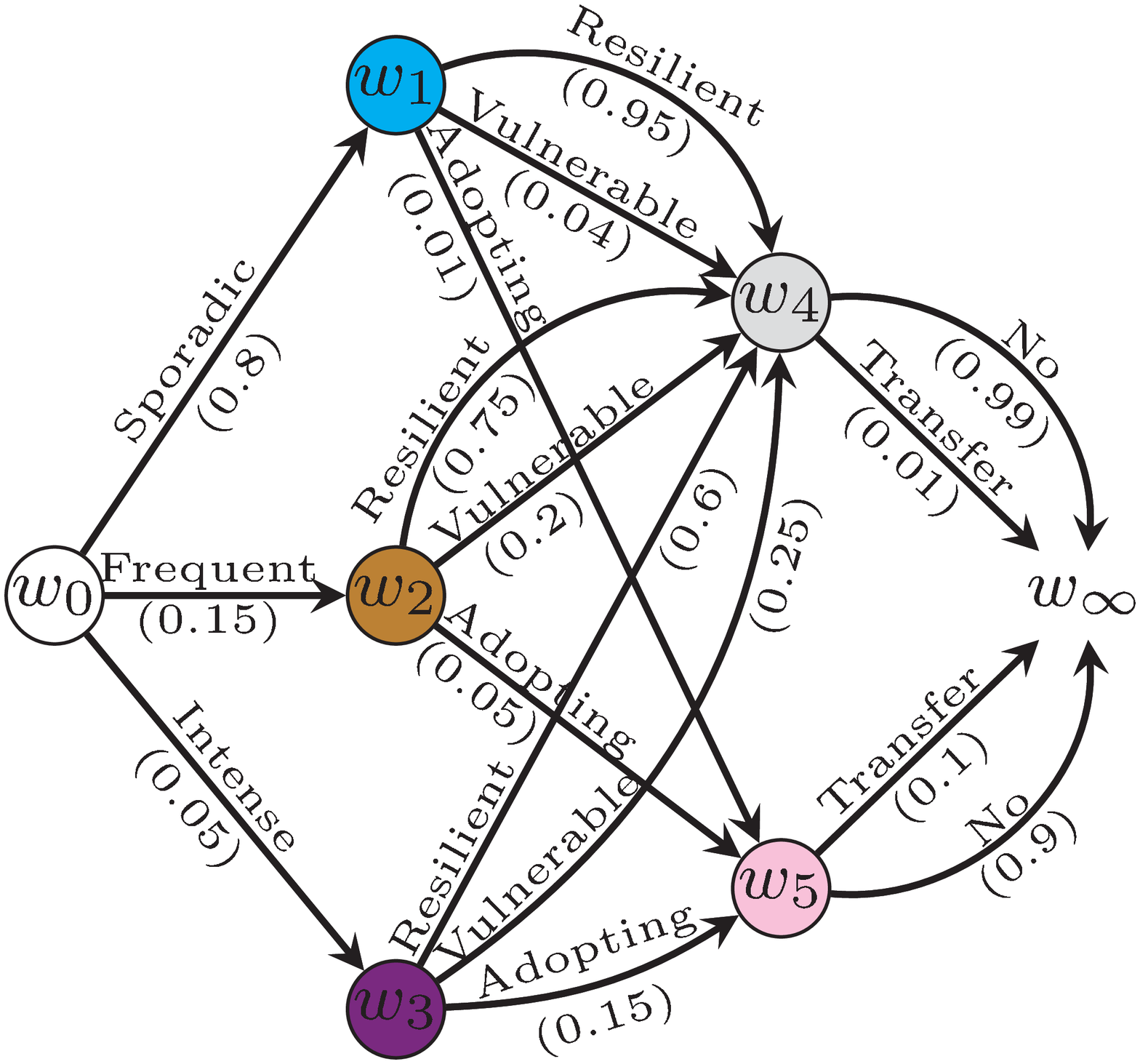}
			\end{flushright}
			\vspace{20pt}
			\caption{CEG~$\mathbb{C}$\label{fig:CEG_Radicalisation} $\!\!\!\!\!\!\!\!\!\!\!\!\!\!\!\!\!\!$}
		\end{subfigure} 
	\end{center}
	\vspace{-17pt}
	\caption{Staged Tree and CEG associated with Example~\ref{ex:Radicalisation_Static_3variables}. The conditional probability of an event given that an inmate is in a position corresponding to a particular vertex in the CEG is shown in parentheses. \label{fig:ST_CEG_Radicalisation}}
	%\vspace{0pt}
\end{figure}

\emph{Finally, a CEG model requires domain experts to elicit the conditional probabilities of an event that may immediately happen to an unit given his actual position in the process. In the CEG depicted in Figure~\ref{fig:CEG_Radicalisation},  domain experts need to quantify six conditional probabilities, one for each vertex~$w_i$, $i=0,\ldots,5$. For example, they may define the conditional probability of an inmate been transferred given that he is in the position~$w_5$ as~$0.1$. Note that the vertex~$w_5$ represents inmates who are adopting radicalisation regardless of their levels of social interaction with recruiters.}

\emph{This CEG model translates into a BN model. We first have to identify a set of suitable random variables, which is not necessarily unique. From the CEG (Figure~\ref{fig:CEG_Radicalisation}) we can discern three categorical variables:  variable~Network~($N$) describing the categories~$s$, $f$ and~$i$; variable~Radicalisation ($R$) signalising the levels~$r$, $v$ and~$a$; and variable~Transfer ($T$) distinguishing between the classes~$n$ and~$t$. A method to construct random variables from the CEG topology is formally presented in~\cite{Smith.Anderson.2008} and~\cite{Collazo.etal.2018}. The CEG model then enables us to read two conditional independence statements: $T$ is independent of~$N$ given~$R$; and $T$ is independent of~$N$ given that $R$~does not assume value~$a$. Figure~\ref{fig:BN_Radicalisation} depicts a possible BN to represent this process. Note that without introducing new random variables this BN cannot represent graphically the context-specific hypothesis above associated with the variable~T given $(R\neq a)$. This kind of asymmetric conditional independences can only be expressed inside its conditional probability tables or through some supplementary semantics with further context-specific information, such one provided by context-specific BNs.}

\vspace{0pt}
\begin{figure}[t]
	\begin{center}
		\includegraphics[scale=0.25,angle=-90,origin=c,trim=30 130 400 0]{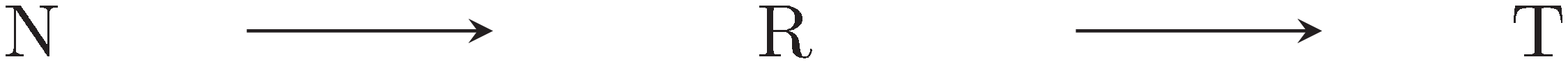}
	\end{center}
	\caption{The BN associated with Example \ref{ex:Radicalisation_Static_3variables} \label{fig:BN_Radicalisation}}
	\vspace{0pt}
\end{figure} 
\end{myExample}

There has been previous interest in developing dynamic classes of CEGs. In~\cite{Freeman.Smith.2011b} the authors proposed a CEG multi-process model where different cohorts of units entering the system at different discrete time points are observed during only one time interval. Although trees are identical across all the cohorts the transition probabilities are allowed to shift dynamically. In contrast, a DCEG is designed for a different purpose. It describe how a single cohort of units who arrive in the system simultaneously might evolve over successive time points.

To pursue our objectives, in Section~\ref{sec:Factorisation_Tree} we develop a new object approach to elicit a process using an infinite tree. An object compactly depicts an event tree that is defined according to the temporal structure and domain information associated with the corresponding  process. This approach enables us to incorporate time-invariant information in a DCEG model and to split the modelling task between different domain expert teams~\cite{Collazo.Smith.2018}. By doing this we ensure the consistency of the composite model. The importance of this last issue is discussed, for example, in \cite{Leonelli.Smith.2015, Smith.etal.2015}.  Here we will also introduce a formal framework to embed a probabilistic map on an infinite tree and the tree-based objects. This development draws on earlier converted technical advances concerning rather different structures called probabilistic automata~\cite{Segala.1995,Stoelinga.2002}.

After discussing some topological concepts on trees, we proceed to give a general representation of a finite DCEG in terms of a particular graphical periodicity and time-homogeneity in Section~\ref{sec:DCEG_class}. This characterisation coupled with tree-based objects stimulates us to introduce the $N$~Time-Slice Dynamic Chain Event Graph~($N\text{T-DCEG}$). We then formalise a link between a Markov state-transition diagram and an~$N\text{T-DCEG}$. In Section~\ref{sec:CEG_NTDCEG} we explain how a particular set of CEGs can be associated with a DCEG model. This connection enables us to propose an algorithm to construct $N$T-DCEG models using the tree-based objects. In Section~\ref{sec:conditional_independence} we show how the implicit conditional independence relationships encoded in an $N\text{T-DCEG}$ can be read from its representation. We conclude the paper with a short discussion. 

%%%%%%%%%%%%%%%%%%%%%%%%%%%%%%%%%%%%%%%%%%%%%%%%%%%%%%%%%%%%%%%
\section{Describing a process using a Tree}
\label{sec:Factorisation_Tree}

In this section, we present the topological concepts associated with the first two of those steps taken to obtain a DCEG: the elicitation of an event tree and its subsequent embellishment into a staged tree. This motivates us to propose a new way to construct an event tree and a staged tree based on objects. We also explore periodicity in event trees and time-homogeneity in staged trees. These concepts are very important because they play a key role in the construction of a finite DCEG graph as we will discuss in Section~\ref{sec:DCEG_class}. We are then able to formally articulate the objects with these two ideas and time-invariant information. 

%%%%%%%%%%%%%%%%%%%%%%%%     
\subsection{An Event Tree}
\label{subsec:Event_Tree}

Recall that in a direct graph~$\mathbb{G}=(V_{\mathbb{G}},E_{\mathbb{G}})$ a \emph{walk} of length~$L$ is a sequence of vertices $(v_{i_0},\ldots,v_{i_L})$ such that every edge $(v_{i_k},v_{i_{k+1}})$, ${k=0,\ldots,L\!-\!1}$, pertains to $E_{\mathbb{G}}$. A walk~$(v_{i_0},\ldots,v_{i_L})$ that all vertices~$v_{i_j}$ in~$\{v_{i_k}, k=0,\ldots,L\}$ are distinct is called a \emph{path$\!$}~\cite{Cowell.etal.2007,Diestel.2006}. Let the parent set ${pa(v)\!=\!\{v'\!\!\in\! V_{\mathbb{G}}; (v'\!,v) \!\in\! E_{\mathbb{G}}\}}$ be the set of vertices from which a vertex~$v$ unfolds. Also let the child set~$ch(v)=\{v' \in V_{\mathbb{G}}; (v,v') \in E_{\mathbb{G}}\}$ be the set of vertices that unfold from~$v$.
An event tree \citep{Shafer.1996} is formally defined below. 

\begin{myDef}
	\label{def:event_tree}
An \textbf{event tree} $\mathcal{T}=(V,E)$ is a rooted directed tree where all edges are labelled and where the directionality of the tree guarantees that all non-root vertices has only one parent  and descend from the root vertex.  If $V$ is a finite set the event tree is said to be \emph{finite}, otherwise it is said to be \emph{infinite}. In this paper, every vertex in V has a finite set of children.	
\end{myDef}

By focusing on the qualitative description of a process, an event tree is an important tool because it provides a framework around which a technical expert can explain the process that needs to be statistically modelled. For example, each path in the tree describes the various possible sequences of events a particular unit can experience along the process under analysis. 

There may be two types of vertices in an event tree~$\mathcal{T}=(V,E)$, a situation and a leaf. A \emph{situation}~$s_i$ is associated with a vertex~$v_i$ in~$V$ that has at least one child and a \emph{leaf}~$l_j$ corresponds to a vertex~$v_j$ with no child. The labelled edges in~$E$ that emanate from a vertex~$v_i$ in $V$ characterise events that can happen to a unit in a particular situation $s_i$ along the process. So, a situation~$s_i$ represents a transitional state. In contrast, a leaf~$l_j$ corresponds to a possible terminating state of the process. However, a situation~$s_i$ and a leaf~$l_j$ are both characterised as the result of a sequence of events depicted on the labelled edges of the path from the root vertex to its corresponding vertex $v_i$ and $v_j$, respectively. 

A \emph{floret} $\mathcal{F}l(s_i)=(V(s_i),E(s_i))$ is a directed subgraph of~$\mathcal{T}$, which is rooted at the situation~$s_i$ and whose other vertices are the children of $s_i$ in $\mathcal{T}$. Formally, $V(s_i)=\{s_i\} \cup ch(s_i)$ and $E(s_i)=\{e \in E; e=(s_i,v), v\in ch(s_i)\}$. This summarises the possible events that can unfold immediately from $s_i$. It therefore depicts the states that a unit can reach in one step once it has arrived at $s_i$. These concepts are illustrated in the example below.

\begin{myExampleCont1}
\emph{
In Figure \ref{fig:FiniteEventTree_Radicalisation} if a prisoner is at situation $s_7$, the tree represents the hypothesis that he is unlikely to be radicalised despite holding frequent social interactions with recruiters. The floret $\mathcal{F}l(s_7)$ associated with situation $s_7$ is depicted in bold and describes graphically how this process may unfold from $s_7$: here that the individual may be transferred or not.}
\end{myExampleCont1}

Unfortunately, an event tree may quickly become large, both in width and depth. In practice to keep it tractable and useful, it is therefore important to develop a systematic framework enabling us to represent it more compactly. This is particularly important when we need to handle processes supported by infinite trees as it is the focus of this paper. To obtain some insights about how to do it, consider first the example below.

\begin{myExampleCont1}
\emph{Recall the radicalisation process described in Figure~\ref{fig:FiniteEventTree_Radicalisation}. Note that the multiple states of this process can be graphically hidden using the big rectangle~$\delta(\mathcal{T})$ in Figure~\ref{fig:FiniteEventTree_Radicalisation} that contains all situations between the root vertex and the set of leaf vertices. In doing this we are able to represent an event tree as a special tree object~$\Delta(\mathcal{T})$. This emphasises only the starting situation of the process (situation~$s_0$) and its possible outcomes (leaves~$l_i$, $i=1,\ldots,18$). To illustrate this convention the event tree of Figure~\ref{fig:FiniteEventTree_Radicalisation} is schematically depicted by the \emph{event tree object}~$\Delta(\mathcal{T})$ in Figure~\ref{fig:Symbol_FiniteEventTree_Radicalisation}. In this case the rectangular vertex~$\delta(\mathcal{T})$ summarily represents a forest graph whose components are the florets associated with positions~$s_1$, $s_2$ and~$s_3$.}
\end{myExampleCont1}

\begin{figure}[h!]
\centering
$\qquad \; \;$
\begin{minipage}{.3\textwidth}
  \centering
  \includegraphics[scale=0.23,angle=-90,origin=c,trim=0 150 -70 -50] {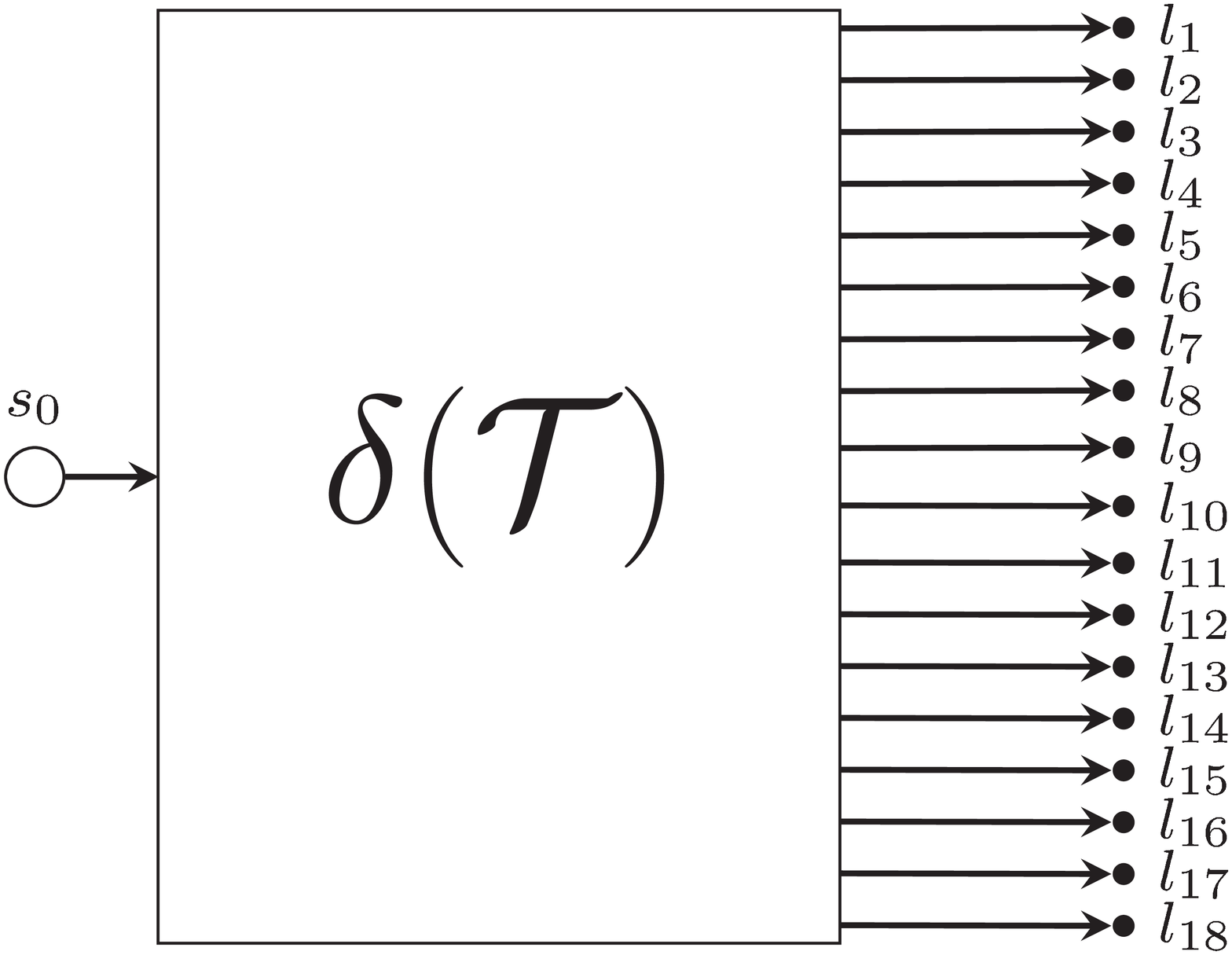}
 \caption{The \emph{event tree object} $\Delta(\mathcal{T})$ that symbolises the Event Tree associated with Example \ref{ex:Radicalisation_Static_3variables}. \label{fig:Symbol_FiniteEventTree_Radicalisation}}
\end{minipage}%
\begin{minipage}{.1\textwidth}
$\qquad$
\end{minipage}
\begin{minipage}{.5\textwidth}
  \centering
  \includegraphics[scale=0.315,angle=-90,origin=c,trim=0 150 -70 0] {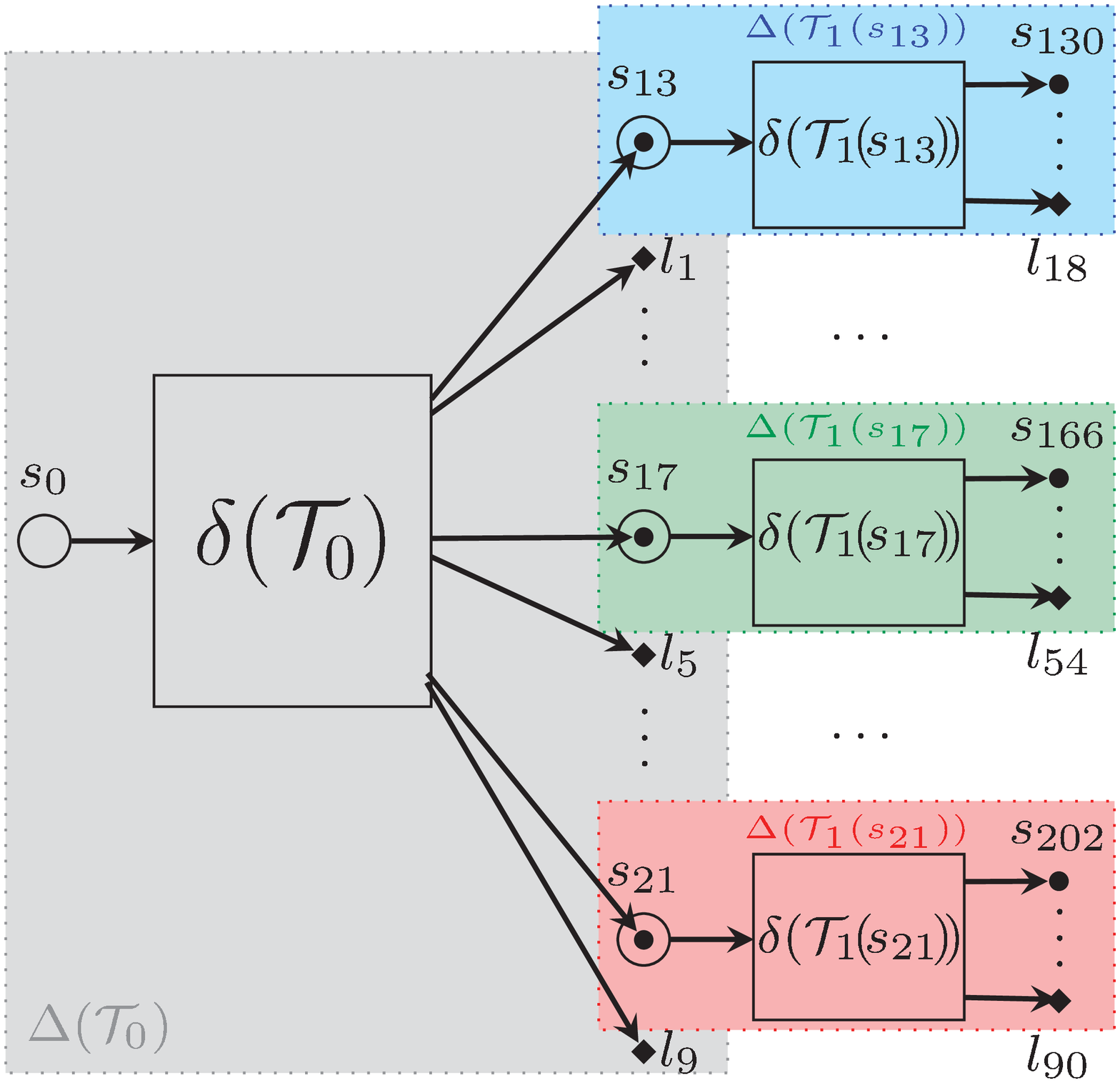}
 \caption{The Event Tree~$\mathcal{T}_1$ associated with Example~\ref{ex:Radicalisation_Dynamic_3variables} that is depicted using event tree objects, each of which is depicted by a colourful rectangle. \label{fig:ObjectOriented_InfiniteEventTree_Radicalisation_3var}}
\end{minipage}
\end{figure}

Let $l(\mathcal{T})$ be the set of leaf nodes of an event tree~$\mathcal{T}$ and $V(\Delta)$ be the interface set constituted by the root and leaf nodes of $\mathcal{T}$. Drawing some analogies with Object-Oriented BNs \cite{Koller.Pfeffer.1997,Bangso.Wullemin.2000} in terms of encapsulating information using objects that can be connected among themselves only via interface nodes we now define a finite event tree object as in Definition \ref{def:tree_object}. 

\begin{myDef}
\label{def:tree_object}
A \textbf{finite event tree object} $\Delta(\mathcal{T})$ is composed by a single input state~$s_0, s_0 \in \mathcal{T}$, and a set of output states~$l(\mathcal{T})$. It compactly depicts a finite event tree~$\mathcal{T}$ by a rectangular vertex~$\delta(\mathcal{T})$, the point vertex set $V(\Delta)$ and a set of direct edges $E(\Delta)=\{(s_0,\delta)\} \cup \{(\delta,s);s\in l(\mathcal{T})\}$, such that:  
\begin{enumerate}[nosep]
\item the initial node $s_0, s_0 \in \mathcal{T}$, is upstream of the rectangle $\delta(\mathcal{T})$;
\item all leaf nodes in $l(\mathcal{T})$ are downstream of the rectangle $\delta(\mathcal{T})$; and
\item the rectangular vertex~$\delta(\mathcal{T})$ represents a forest graph \cite{Diestel.2006} whose components are the event subtrees of~$\mathcal{T}(s_i)$ that unfold from each situation~$s_i$, ${s_i \in ch(s_0)}$, until reaching a subset of situations in
\vspace{-9pt}
$$\{s_j;s_j \in pa(l_i) \text{ for some } l_i \in l(\mathcal{T})\}.\vspace{-5pt}$$
\vspace{-19pt}
\end{enumerate}
If $\mathcal{T}$ is a tree whose vertex set is composed only by a root vertex and the leaf vertices, then the rectangular vertex~$\delta(\mathcal{T})$ represents an empty graph.
\end{myDef}

An event tree object is a process-driven object. Analogous to a compact graphical representation of a BN object \cite{Bangso.Wullemin.2000}, it hides the unfolding of the process and keeps visible only the interface set $V(\Delta)$ which represents the initial state and the possible final states  of the process. The interface set enables us to combine together sub-processes that are components of an ongoing process. 

By doing this it is possible to unfold the same event tree object~$\Delta(\mathcal{T})$ from different leaf vertices~$l_i'$ of a given event tree~$\mathcal{T}'$. When we do this, the node~$s_0$ in~$\Delta(\mathcal{T})$ plays the role of an input vertex  that assumes the value corresponding to the state of a particular unit at a leaf~$l_i' \in l(\mathcal{T}')$. The leaf nodes~$l_i$ in~$\Delta(\mathcal{T})$ are the set of output states that a unit at~$l_i'$ can arrive at after developing according to the local process depicted by~$\Delta(\mathcal{T})$, i.e. the logical concatenation of event propositions representing by the states~$l_i'$ and~$l_i$. In this sense, the rectangular vertex~$\delta(\mathcal{T})$ can be interpreted as a symbolic transformation of states associated with a particular local process. 

To model a longitudinal process, a useful strategy is to just explore particular types of domain information over time. It is straightforward to capture this structure by constructing an event tree using process-driven objects defined according to the time-slices and the unfolding paths of the event tree. Our idea is to construct subtrees corresponding to a subprocess at time-slice~ $t$ given a particular path in the event subtree that describes the development up to the current time point. Using such a representation, parallel panels of experts can take part in the construction of the whole event tree~\cite{Collazo.Smith.2018}. This composite then integrate the domain beliefs coherently and will be compactly presented in terms of objects. For technical consistency, henceforth we assume that only a finite number of events may happen over each time-slice associated with the development of a process.

So, take a situation $s_i$ at any time-slice before the beginning of interval~${t+1}$. Let $\mathcal{T}_{t}(s_i)$ denote the finite tree that unfolds from $s_i$ and stops at the end of interval~$t$. Now write $\mathcal{T}_t \equiv \mathcal{T}_t(s_0)$ and let $\mathcal{T}(s_i)$ be the whole event tree that unfolds from $s_i$. When $\mathcal{T}(s_i)$ is an infinite tree, we sometimes write $\mathcal{T}_\infty(s_i)$ to highlight this fact. The unfolding process in time-slice $t$ is then represented by a collection of event subtrees $\mathcal{F}\!o_t=\{\mathcal{T}_t(s_i); s_i \in l(\mathcal{T}_{t-1})\}$. This collection then constitutes a forest graph whose components are given by $\mathcal{T}_t(s_i)$. Here we use the convention that in an event tree a terminating event is indicated by a diamond shape vertex. This framework is illustrated in the example below. 

\begin{myExample}
	\label{ex:Radicalisation_Dynamic_3variables}
	\emph{
		Radicalisation has a psychosocial dynamic, which is naturally modelled as a process developing over time \cite{Hannah.etal.2008, Neumann.2010, Silke.2011,Christmann.2012,Moghaddam.2005,Silber.Bhatt.2007,Precht.2007}.
		In this setting, a dynamic model enables modellers and domain experts to construct a more accurate and well-defined representation of the process.
		Here suppose that the prison manager tells us that the counts of social interactions, degrees of radicalisation and prison transfers described in Example~\ref{ex:Radicalisation_Static_3variables} are recorded weekly. Being an isolated environment, he thinks that a change in the underlying mechanisms of this process only happens rarely in the short to medium term: the very recent events are the main psychosocial drivers of the prison population. In this scenario, it is plausible that the time-homogeneous and 1-Markov conditions might hold. Hence, as we will formally show in Section~\ref{sec:CEG_NTDCEG}, it is not necessary to elicit the infinite event tree but only the event tree associated with the first two time-slices of this dynamic.
	}
	
	\emph{
		In Figure~\ref{fig:ObjectOriented_InfiniteEventTree_Radicalisation_3var} the event tree $\mathcal{T}_1$ depicts the whole set of events that can unfold from the initial situation $s_0$ to the end of time-slice~$1$. Note that the process terminates at leaves~$l_i, i=1,\ldots,90$, because a prisoner is transferred. Also observe that in this particular example, although this is not a necessary feature in general, every process-driven object happens to be topologically identical to the one depicted by the finite tree $\mathcal{T}$ in Figure~\ref{fig:FiniteEventTree_Radicalisation}.
	}

	\emph{
		The finite tree~$\mathcal{T}_1(s_{13})$ summarises what can happen at time-slice~$1$ to a prisoner who keeps sporadic social contacts with identified extremist recruiters, is resilient to the radicalisation process and has been not transferred at the initial interval. In contrast, the infinite event tree $\mathcal{T}_\infty(s_{13})$ rooted at situation $s_{13}$ describes all possible events that can happen to this type of prisoners. The forest
		\vspace{-5pt}
		$$\mathcal{F}\!o_1=\{\mathcal{T}_{1}(s_i); i=13,\ldots,21\} \cup \emptyset \vspace{-5pt}$$
		then depicts how radicalisation, transfer and network events associated with a prisoner can unfold at time-slice $1$. Observe that the empty set is included in $\mathcal{F}\!o_1$ to stress that the process terminates at leaves ${l_i, i=1,\ldots,9}$.
	}
\end{myExample}

Definition~\ref{def:infinite_tree_object} introduces a formal object that can provide a compact and formal depiction of the type of infinite event tree we need. This is particularly useful when analysts and experts adopt a top-down approach to model a dynamic process with time-invariant aspects using an event tree. In this case, an infinite event tree object stimulates them to explore alternative ways of encapsulating information and working in parallel with different collaborative groups.        

\begin{myDef}
	\label{def:infinite_tree_object} 
	An \textbf{infinite event tree object} $\Delta(\mathcal{T}_\infty)$ represents an infinite tree. An infinite tree object differentiates graphically from that defined for a finite tree in two aspects. First, the leaf vertices of a finite tree are replaced by a single vertex labelled by a symbol $\infty$. Second, this vertex $\infty$ is connected to the rectangular vertex~$\delta(\mathcal{T}_\infty(s_i))$ by a dashed line.
\end{myDef}

In fact, as when building objects elicited during the construction of an object-oriented BNs~\cite{Bangso.Wullemin.2000,Neil.etal.2000}, \emph{tree} objects can also be defined using different domain aspects other than the overarching time-slice division. This property can be especially useful for incorporating time-invariant aspects of units observed in a system into our models. Henceforward, let  $\mathcal{T}_{-1}$ be a finite event tree that represents this time-invariant information, where each root-to-leaf path in~$\mathcal{T}_{-1}$ characterises a particular type of units. To model a process that needs to distinguish between different types of units, we first need to elicit an object $\Delta(\mathcal{T}_{-1})$ that encapsulates this aspect. Only afterwards do we plug-in the event tree objects corresponding to processes that unfold from each leaf of~$\mathcal{T}_{-1}$ as presented above.  This framework is illustrated in the example~\ref{ex:Radicalisation_Dynamic_Covariate}.

\begin{myExample}[Extended Dynamic Radicalisation Process]
	\label{ex:Radicalisation_Dynamic_Covariate} 
	\emph{
		Suppose that Example~\ref{ex:Radicalisation_Dynamic_3variables} refers to a British prison. Assume that all previous conditions continue to hold. We would like now to control the process in this prison for prisoners' previous conviction (Yes or No) and nationality (British or Foreign). To represent this dynamic using an event tree is straightforward as showed in Figure~\ref{fig:ObjectOriented_InfiniteEventTree_Radicalisation_Covariate}.}
	
	%\vspace{-9pt}
	\begin{figure}[ht]
		\begin{center}
			\includegraphics[scale=0.35,angle=-90,origin=c,trim=0 30 0 -130]
			{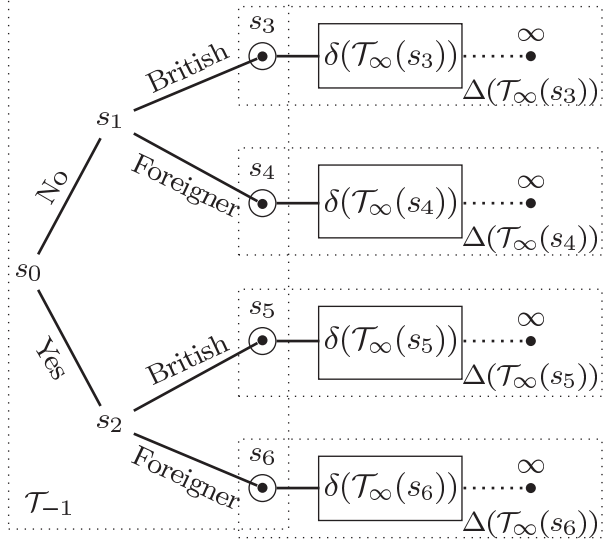}
		\end{center}
		\caption[Infinite tree with time-invariant variables]{The infinite Event Tree associated with the dynamic radicalisation process controlled by two time-invariant variables: Previous Convection and Nationality. The situation $s_0$ corresponds to the variable Previous Conviction and the situations $s_1$ and~$s_2$ are associated with the variable Nationality. The objects~$\Delta(\mathcal{T}_\infty(s_i))$, ${i=3,\ldots,6}$, represent identical infinite event trees.  A dotted rectangle establishes the limit of a particular graph. \label{fig:ObjectOriented_InfiniteEventTree_Radicalisation_Covariate}}
		\vspace{-15pt}
	\end{figure}
	
	\emph{
		Note that the leaves of~$\mathcal{T}_{-1}$ characterise four type of inmates: ${s_1\text{ -- British}}$ non-convicted inmates; $s_2$ -- Foreigner non-convicted inmates; $s_3$ -- British~convicted inmates; and $s_4$ -- Foreigner convicted inmate. The objects $\Delta(\mathcal{T}_\infty(s_i))$, ${i=3,4,5,6}$, represent infinite event trees. It is assumed here that these infinite trees are identical although this assumption could be relaxed.}
	
	\emph{
		In more complex real-world scenarios, this framework also enables us to conduct a distributed model construction that composes coherently the domain information~\cite{Collazo.Smith.2018}. For instance, we could organise distinct teams that would be in charge of modelling the corresponding processes associated with prisoners who has previous criminal charges or not. Subsequently the results could be unified using these objects that could be further split to allow the identification of finer commonalities between the processes ($\mathcal{T}_\infty(s_1)$ and $\mathcal{T}_\infty(s_2)$) in these two prison sub-populations.}
\end{myExample} 

Although some care is needed it is nevertheless important to formalise mathematically the object-recursive approach developed above. So, take a finite event tree~$\mathcal{T}$ and encapsulate it into an event tree object $\Delta(\mathcal{T})$. Denote by~${\Gamma=\{\Delta(\mathcal{T}_i)\}}$ a set of input event tree objects that unfold from~$\mathcal{T}$ in a observed process. To connect these event tree objects using their interface sets of nodes, we have to partition the leaf nodes of~$\Delta(\mathcal{T})$ according to the output states provided by its corresponding subprocess. For this purpose, let ${\Upsilon(\mathcal{T})=\{\Upsilon_k; k=1,\ldots,n\}}$, where $\Upsilon_k=\{l_{k_i}; i=1,\ldots,n_k\}$ is a partition of the leaf vertex set of $\mathcal{T}$. Note that each set~$\Upsilon_k$ needs to have a different meaning in terms of the subsequent unfolding process dynamic that is identical between the output states clustered in~$\Upsilon_k$.

For example, using the event tree object in Figure~\ref{fig:Symbol_FiniteEventTree_Radicalisation} we can  choose
\vspace{-5pt}
$$\Upsilon\!(\mathcal{T}) \!=\! \{\Upsilon_1 \!=\! \{l_{1},l_{2},\ldots,l_{6}\},\! \Upsilon_2 \!=\! \{l_{7},l_{9},\ldots,l_{17}\},\! \Upsilon_3 \!=\! \{l_{8},l_{10},\ldots,l_{18}\}\}. \vspace{-5pt}$$
Note that each partition identifies prisoners according to the particular set of unfolding events. Remember that $\delta(\mathcal{T})$~in~$\Delta(\mathcal{T})$ is depicted in Figure~\ref{fig:FiniteEventTree_Radicalisation}. So, $\Upsilon_1$~characterizes  prisoners who have few social contacts with extremist recruiters, whilst $\Upsilon_3$ and $\Upsilon_2$ are  defined by prisoners with frequent or intense social contact with extremist ideologists and, respectively, were transferred or not. Prison managers would validate this partition if and only if they identified different subprocesses unfolding from states clustered by each set~$\Upsilon_k,$ ${k=1,2,3}$. Proceed to define the merging operation $\Delta(\mathcal{T}) \uplus_{h} \Gamma$ between a finite event tree object $\Delta(\mathcal{T})$ and a set of event tree objects $\Gamma$ according to a map $h: \Upsilon(\mathcal{T}) \to \Gamma$.

\begin{myDef}
\label{def:merging_operation}
The \textbf{merging operation} $\boldsymbol{\Delta(\mathcal{T}) \uplus_{h} \Gamma}$ gives as its output an event tree object $\Delta(\mathcal{T}_+)$. This unfolds the event tree object $h(\Upsilon_k)$ from each leaf node $l_i \in \Delta(\mathcal{T})$ such that $l_i \in \Upsilon_k$. If $\mathcal{T}=\emptyset$, then $\Delta(\mathcal{T})=\emptyset$ and $\Gamma$ must be singleton. In this case, $\Delta(\mathcal{T}) \uplus_{h} \Gamma= \Delta(\mathcal{T}_*)$, where $h:\emptyset \to \Gamma=\{\Delta(\mathcal{T}_*)\}$.
\end{myDef}

Now for every time-slice $t$ take a partition $\Upsilon(\mathcal{T}_t)$ associated with the finite event tree object $\Delta(\mathcal{T}_t)$ together with a set of finite event tree objects
$\Gamma_t=\{\Delta(\mathcal{T}_{t_i});$ ${i=1,\ldots,n_t}\}$.
Next define any map $h_t: \Upsilon(\mathcal{T}_t) \to \Gamma_t$. Note that it is possible that $\mathcal{T}_{t_i}=\emptyset$. Then, for every time-slice $t=0,1,\ldots$, we have that
\begin{equation}
\label{eq:recursive_structure_tree}
\Delta(\mathcal{T}_{t+1})=\Delta(\mathcal{T}_{t})\uplus_{h_{t}} \Gamma_{t},
\end{equation}
where $\Delta(\mathcal{T}_{-1})=\emptyset$ and $h_{-1}:\emptyset \to \Gamma_0=\{\Delta(\mathcal{T}_0)\}$ whenever there is no time-invariant information. Observe that this merging operation corresponds to adding the process-driven objects $\Delta(\mathcal{T}_1(s_i)), i=13, 14,\ldots,21$, to the leaf vertices of $\Delta(\mathcal{T}_0)$ in the way depicted in Figure~\ref{fig:ObjectOriented_InfiniteEventTree_Radicalisation_3var}.

More formally we define an infinite event tree object~$\Delta(\mathcal{T}_\infty)$ as the direct limit~\cite{Ash.2007}
\begin{equation}
\Delta(\mathcal{T}_\infty)=\varinjlim \Delta(\mathcal{T}_t)
\label{eq:infinite_tree}
\end{equation}
of the system $\{\Gamma,f(i,j);i,j\!=\!-1,0,1,\ldots\}$, where $\Gamma\!=\!\{\Delta(\mathcal{T}_t); t\!=\!-1,0,1,\ldots\}$ and the morphism $f: \Delta(\mathcal{T}_i) \to \Delta(\mathcal{T}_j), j \geq i$,  is such that
\begin{equation}
f(\Delta(\mathcal{T}_j))=\Delta(\mathcal{T}_{i})\uplus_{h_{i}}\Gamma_{i}\ldots \uplus_{h_{j-1}} \Gamma_{j-1}.
\label{eq:morhism}
\end{equation}

An important type of infinite event trees called the Periodic Event Tree is defined below. This family of event trees support processes commonly found in real-world applications and enable us to embed various time-homogeneity assumptions. An infinite event tree object whose underlying event tree is periodic is defined below. Its construction requires us to define only two finite event tree objects. In Section~\ref{subsec:NT-DCEG} this infinite tree object family is used to define a new class of DCEGs. For a more extensive discussion of periodicity in probabilistic trees, see \cite{Peres.1999}.

Let $\Lambda(\mathcal{T})\!=\!\{\lambda \subset \mathcal{T}; \lambda \text{ is a root-to-leaf or an infinite path}\}$ denote the set of paths of a tree $\mathcal{T}$. Next let $s(t)$ denote a situation that happens in time-slice~$t$. Finally, let~$\boldsymbol{\tau}(\lambda)=(\tau_i(\lambda))_{i}$ denote the ordered sequence of time-slices, such that $\tau_i(\lambda)$ is the time-slice associated with the $i$th event that happens along the path~$\lambda$. For instance, define $\lambda$ as the $s_0 $-to-$s_{130}$ path  in Figure~\ref{fig:ObjectOriented_InfiniteEventTree_Radicalisation_3var}. It then follows that $\boldsymbol{\tau}(\lambda)=(0,0,0,1,1,1)$ since the first three events happen in time-slice~$0$ and the subsequent three events happen in time-slice~$1$.

\begin{myDef} 
\label{def:periodic_event_tree}
An infinite event tree is a \textbf{Periodic Event Tree after time}~$\boldsymbol{T}$ ($\text{PET-}T$), $T=0,1,\ldots$, if and only if for every situation~$s_a(t_a)$, ${t_a= T+1, T+2, \ldots}$, there is a situation~$s_b(t_b), {t_b= 0,1,\ldots,T}$, such that there exists a bijection
\vspace{-9pt}
\begin{equation}
\label{eq:bijection_periodicity_tree}
\phi(s_a,s_b):\Lambda(\mathcal{T}_\infty(s_a)) \to \Lambda(\mathcal{T}_\infty(s_b)),
\vspace{-7pt}
\end{equation}
satisfying the following three conditions:
\begin{enumerate}[nosep]
\item \emph{Global condition} - the ordered sequence of events in a path $\lambda \in \Lambda(\mathcal{T}(s_a))$ equals the ordered sequence of events in the path $\lambda'=\phi(s_a,s_b)(\lambda)$, $\lambda' \in \Lambda(\mathcal{T}(s_b))$.
\item \emph{Local condition} - for every path $\lambda \in \Lambda(\mathcal{T}(s_a))$ we have that\vspace{-7pt}
$${\tau_i(\lambda)=\tau_i(\lambda')+(t_a-t_b)}, {i\in \mathcal{I}(\lambda)}, \vspace{-5pt}$$ where $\lambda'=\phi(s_a,s_b)(\lambda), \lambda' \in \Lambda(\mathcal{T}(s_b))$.
\item \emph{Time-invariant condition} - if $\mathcal{T}_{-1} \ne \emptyset$, then $s_a$ and $s_b$ must unfold from the same leaf node of $\mathcal{T}_{-1}$.
\end{enumerate}
\end{myDef}

\begin{myDef} 
	\label{def:periodic_event_tree_generated_by_T}
An infinite event tree object~$\Delta(\mathcal{T}_\infty)$ is said to be a \textbf{Tree Object Generated by finite event trees}~$\boldsymbol{\mathcal{T}_{-1}}$ and~$\boldsymbol{\mathcal{T}}$ (TOG($\mathcal{T}_{-1},\mathcal{T}$)) if and only if we can construct it using equation~\ref{eq:infinite_tree} given that 
$\Upsilon(\mathcal{T}_{-1})=\{\Upsilon_{-1,1}\}$, ${\Gamma_{-1}=\{\Delta(\mathcal{T})\}}$ and ${h_t:\Upsilon(\mathcal{T}_{-1}) \to \Gamma_{-1}}, h_{-1}(\Upsilon_{-1,1})=\Delta({\mathcal{T}})$, for any $\mathcal{T}_{-1}$, and one of the following types of periodic time-slice structure: 
\begin{enumerate}[nosep]
	\item \emph{Type A} - for all $t=0,1,\ldots$, we have that $\Upsilon(\mathcal{T}_t)=\{\Upsilon_{t,1}\}$, ${\Gamma_t=\{\Delta(\mathcal{T})\}}$ and ${h_t:\Upsilon(\mathcal{T}_t) \to \Gamma_t}$, $h_t(\Upsilon_{t,1})=\Delta({\mathcal{T}})$; or
	
	\item \emph{Type B} - for all $t=0,1,\ldots$, we have that
	$\Upsilon(\mathcal{T}_t)=\{\Upsilon_{t,1}\ne\emptyset,\Upsilon_{t,2}\ne\emptyset\}$, ${\Gamma_t=\{\Delta(\mathcal{T}),\emptyset\}}$ and $h_t:\Upsilon(\mathcal{T}_t) \to \Gamma_t, h_t(\Upsilon_{t,1})=\Delta({\mathcal{T}})$ and $h_t(\Upsilon_{t,2})=\emptyset$.
	In this case, $\Upsilon_{t,1}$ is the set of leaf nodes of $\Delta(\mathcal{T}_t)$ whose associated sequence of events at time~$t$ corresponds to some leaf node of $\Delta(\mathcal{T}_0)$ in $\Upsilon_{0,1}$, and ${\Upsilon_{t,2}= l(\Delta(\mathcal{T}_t))-\Upsilon_{t,1}}$.
\end{enumerate}
If we stop the construction of a TOG($\mathcal{T}_{-1},\mathcal{T}$) at time-slice~$t$ we obtain a finite event tree object~$\Delta(\mathcal{T}_t)$ that is said to be a \textbf{Tree Object Generated by finite event trees}~$\boldsymbol{\mathcal{T}_{-1}}$ and~$\boldsymbol{\mathcal{T}}$ \textbf{until time-slice}~$\boldsymbol{t}$ (TOG($\mathcal{T}_{-1},\mathcal{T},t$)).
\end{myDef}

In a TOG($\mathcal{T}_{-1},\mathcal{T}$) type~A there is not a leaf node. Conversely, a process represented by a TOG($\mathcal{T}_{-1},\mathcal{T}$) type B has some of its branches missing at any time-slice and so a unit at the beginning of any time-slice can take a path that leads it to a leaf node. Note that every TOG($\mathcal{T}_{-1},\mathcal{T}$) is supported by a PET-${0}$. For example, looking at~$\mathcal{T}_1$ showed in Figure~\ref{fig:ObjectOriented_InfiniteEventTree_Radicalisation_3var}, it is straightforward to see that the event tree~$\mathcal{T}_\infty$ corresponding to Example~\ref{ex:Radicalisation_Dynamic_3variables} is a ${\text{PET-}0}$ and can be expressed as a TOG($\mathcal{T}_{-1},\mathcal{T}$) classe B, where $\mathcal{T}_{-1}=\emptyset$ and~$\mathcal{T}$ is depicted in Figure \ref{fig:FiniteEventTree_Radicalisation}.

%%%%%%%%%%%%%%%%%%%%%%%%
\subsection[Staged Tree]{A Probability Space and a Staged Tree}
\label{sec:Staged_Tree}

A staged tree is obtained when analysts associate a given event tree with a probability space. To do this, take a sample space defined as the set of paths $\Lambda(\mathcal{T}_\infty)$. As for a finite tree \cite{Shafer.1996}, a situation~$s_i$ in an infinite tree can be associated with a random variable $X(s_i)$ that describes the possible next developments of a process once a unit has arrived at~$s_i$. The state space $\mathbb{X}(s_i)$ of $X(s_i)$ corresponds to the set of labels associated with the emanating edges of $s_i$. Let $\gamma_{ij}$ be the label corresponding to the edge $(s_i,v_j), v_j \in ch(s_i)$. Now, for each situation $s_i \in \mathcal{T}_\infty$, define the primitive probabilities
\begin{equation}
\label{eq:primitive_probability}
\pi(v_j|s_i)=P(X(s_i)=\gamma_{ij}|s_i), v_j \in ch(s_i).
\end{equation}

Let the path-cylinder $\Lambda(v_i)$ be the set of all paths in $\Lambda(\mathcal{T}_\infty)$ that pass through a situation $v_i$ or terminate at a leaf~$v_i$. Also let $\Psi(v_i)$ be the time-ordered concatenation of situations along the root-to-$pa(v_i)$ path. Let $\psi(s,v_i), s \in \Psi(v_i)$, denote the child vertex of $s$ along the root-to-$v_i$ path. From the usual rules of conditioning, we then have that 

\begin{equation}
\label{eq:probability_measure_cilinder}
P(\Lambda(v_i))=
\prod_{s \in \Psi(v_i)} \pi(\psi(s,v_i)|s).
\end{equation}
Using the Extension Theorem (\cite{Feller.1971}, p. 119), we are then able to uniquely extend the probability measure defined in Equation~\ref{eq:probability_measure_cilinder} to the smallest $\sigma\text{-algebra}$ of $\{\Lambda(v_i);v_i\in \mathcal{T}_\infty\}$, the so called path-cylinder $\sigma$-algebra; see also \cite{Segala.1995,Stoelinga.2002}. So our probabilistic model is well specified by a pair $(\mathcal{T}_\infty,\Pi)$, where
\begin{equation}
\label{eq:set_primitive_probability}
\Pi=\{\pi(v|s_i);v \in ch(s_i),s_i \in \mathcal{T}_\infty)\}
\end{equation} 
is the set of all primitive probabilities defined over an infinite event tree $\mathcal{T}_\infty$.

It is possible to define a stage partition slightly differently from the way it is defined in~\cite{Barclay.etal.2015} and~\cite{Smith.Anderson.2008}. A stage continues being a partition set of situations that collects together in a single cluster all those situations whose 1-step unfoldings are exchangeable conditional on the event that a unit is at a situation contained in this stage. However, here there is a further external condition that the bijection defined in equation~\ref{eq:bijection_stage} has to be meaningful in terms of real-world applications. This does not only enable analysts to embellish their models with important domain knowledge but also helpfully restricts the set of valid bijections. Of course, this condition can be ignored if there is not background information that needs to be enforced.

\begin{myDef}  
	\label{def:stage}
	Two situations $s_a$ and $s_b$ are said to be in the same \textbf{stage}~$\boldsymbol{u}$ if and only if the random variables $X(s_a)$ and $X(s_b)$ have the same probability distribution under a bijection
	\begin{equation}
	\label{eq:bijection_stage}
	\phi_u(s_a,s_b):\mathbb{X}(s_a) \to \mathbb{X}(s_b),
	\end{equation}
	where the labels~$\gamma_{ai}$ and~$\gamma_{bj}$, such that $\gamma_{bj}=\phi_u(s_a,s_b)(\gamma_{ai})$, can be assigned the same context sense in natural language. Every stage has a unique colour that differentiates it from others. An \textbf{infinite staged tree} ${\mathcal{ST}_\infty}$ is then obtained when its corresponding event tree $\mathcal{T}_\infty$ is embellished with those colours such that there are infinite many situations that are visited with non-zero probabilities. The partition yielded by stages over the set of situations in an staged tree $\mathcal{ST}_\infty$ is called \textbf{stage structure}.  
\end{myDef}

Definition~\ref{def:staged_tree_object} extends straightforwardly the concept of event tree objects to staged trees. Remember that a stage structure represents a collection of context-specific conditional independence statements that restrict the space of valid probability measures without selecting one. It follows that it is not necessary to have a probability measure in order to construct a staged tree object. This enables modellers and decision makers to focus on the main structural characteristics that drive the process at hand before populating numerically the model with probability distributions. This is somewhat analogous to construct a BN graph without having to elicit its corresponding probability distributions.  

\begin{myDef} 
	\label{def:staged_tree_object}
	A \textbf{staged tree object} is obtained when an event tree object is associated with a valid stage structure. If an event tree object is associated with a valid probabilistic measure, we then obtained a \textbf{probabilistic tree object}.
\end{myDef}

\begin{myExampleCont2}
	\emph{
		Return to the radicalisation process described in Example~\ref{ex:Radicalisation_Dynamic_3variables} and depicted for the first two time-slices in Figure~\ref{fig:ObjectOriented_InfiniteEventTree_Radicalisation_3var}. Now we can show this event tree to the prison manager and use it to elicit the conditional probability statements. 
		Suppose that he then tells us that the hypotheses in Example~\ref{ex:Radicalisation_Static_3variables} are still valid for each time-slice if a prisoner remains in prison. Next he reports that the level of socialisation of an inmate with a potential recruiter at time~$t\!+\!1$ depends only on its previous value observed at time~$t$ given that an inmate remains in prison. He also believes that if we know the level of socialisation of an inmate at time~$t\!+\!1$ and his the degree of radicalisation at time~$t$, then other past events do not bring any relevant information to infer his degree of radicalisation at time~$t\!+\!1$. Furthermore, he explains that the probability of transferring a prisoner at time~$t\!+\!1$ is independent of all set of past events given that the prisoner remains in prison and his degree of radicalisation is known at time~$t\!+\!1$. 
	}

	\emph{
		Finally, assume that the prison manager confirms two generally accepted views on the radicalisation dynamic observed in prisons. One is that a prisoner with an extreme political or religious ideology typically constructs social networks that are unlikely to moderate him and might well reinforce his current beliefs. According to this description, the prisoner's extremist values and judgements implicitly guides his social contacts. This translates into the statement that the degree of an inmate's radicalisation  at time~$t\!+\!1$ is independent of his social networks at time~$t\!+\!1$ given that he adopted radicalisation and remains in prison at previous time~$t$. On the other hand, the conversion of a non-radical prisoner to an extremist ideology could well be driven by his social contacts within the prison regardless of whether he is resilient or vulnerable. Under this reasoning the social network might act on the prisoner's belief system but not the other way around. This supports the hypothesis that the degree of a prisoner's radicalisation  at time~$t\!+\!1$ is independent of its previous value at time~$t$ given that he did not adopt radicalisation and was not transferred at time $t$, and his level of socialisation with potential recruiters is known at time~$t\!+\!1$. 
	}

	\emph{
		Figures \ref{fig:ST_1_s15} and \ref{fig:ST_1_s21} display the staged subtrees for time-slice~$1$ with respect to vulnerable non-transferred prisoners who keep, respectively, sporadic (situation~$s_{14}$) and frequent (situation~$s_{17}$) social contact with extremist recruiters during the initial time-slice. Based on expert's judgement, assume that the primitive probabilities of a prisoner at situations $s_{14}$ and $s_{17}$ are given by: $P(X(s_{14})=s|s_{14})=0.7$, ${P(X(s_{14})=f|s_{14})=0.2}$, ${P(X(s_{14})=i|s_{14})=0.1}$, $P(X(s_{17})=s|s_{17})=0.2$, $P(X(s_{17})=f|s_{17})=0.7$, ${P(X(s_{17})=i|s_{17})=0.1}$.
	}

\begin{figure}[t]
	%\vspace{37pt} 
	%\centering
	\begin{subfigure}{.5\textwidth}
		\begin{center}
			\includegraphics[scale=0.3,angle=0,origin=c,trim=50 0 0 0]{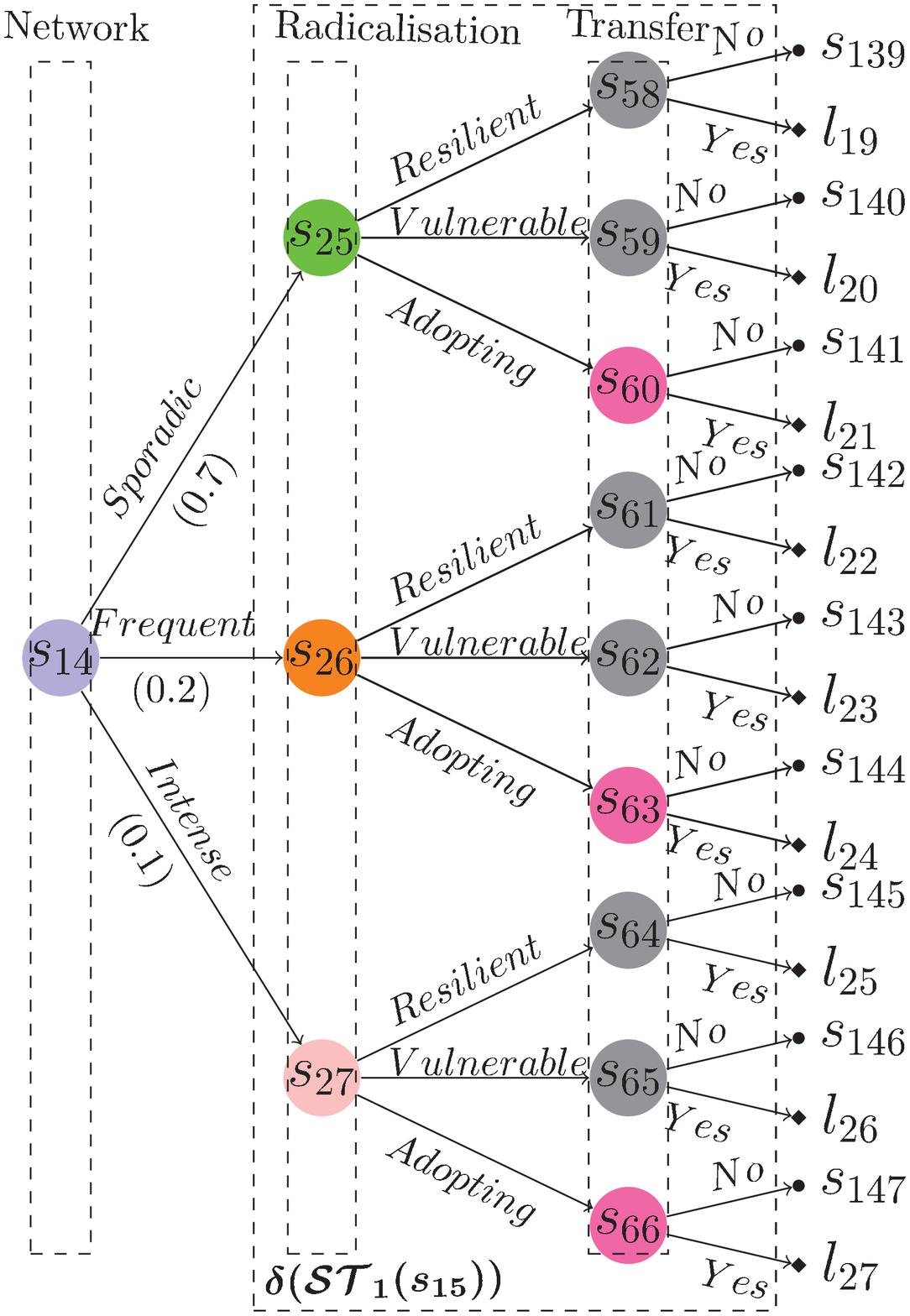}
		\end{center}
		\vspace{-10pt}
		\caption{Staged Subtree $\mathcal{ST}_1(s_{14}) \!\!\!\!\!\!\!\!\!\!\!\!\!\!\!$ \label{fig:ST_1_s15}}
		\vspace{0pt}
	\end{subfigure}%
	\begin{subfigure}{.5\textwidth}
		\begin{center}
			\includegraphics[scale=0.3,angle=0,origin=c,trim=50 0 0 0]{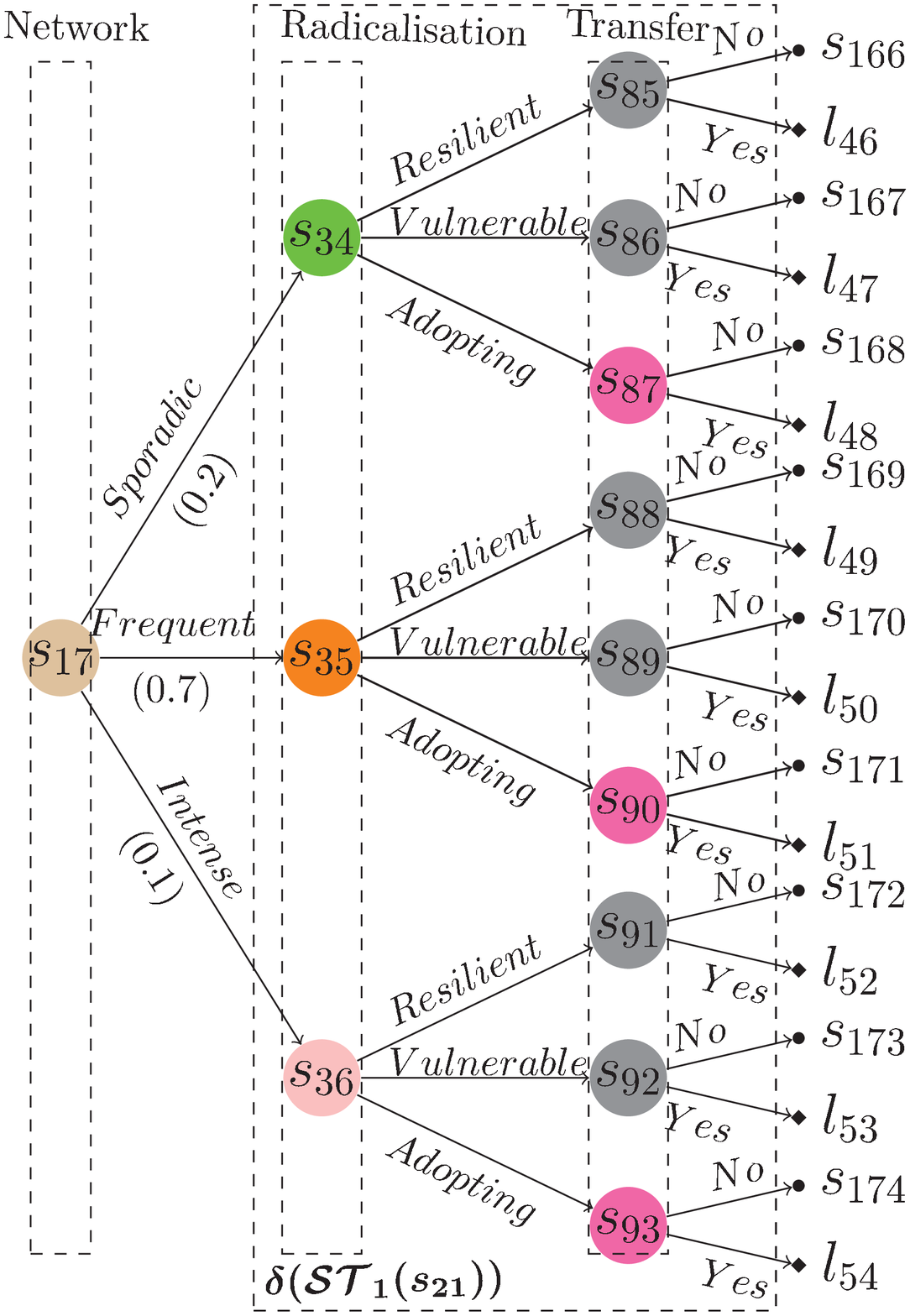}
		\end{center}
		\vspace{-10pt}
		\caption{Staged Subtree $\mathcal{ST}_1(s_{17}) \!\!\!\!\!\!\!\!\!\!\!\!\!\!\!$ \label{fig:ST_1_s21}}
		\vspace{0pt}
	\end{subfigure} %
	\caption{Two Staged Subtrees of time-slice~$1$ corresponding to the dynamic radicalisation process described in Example \ref{ex:Radicalisation_Dynamic_3variables} and partially depicted in Figure~\ref{fig:ObjectOriented_InfiniteEventTree_Radicalisation_3var}.
		The primitive probabilities associated with situations~$s_{14}$ and~$s_{17}$ are shown in parentheses.
		\label{fig:ST_1_Radicalisation_Dynamic_3variables}}
	\vspace{-13pt}
\end{figure}

	\emph{
		Here, for example, since the probability of transfer is conditionally independent of social interactions and does only change for prisoners currently adopting radicalisation, the situations associated with transfer can be coloured using only two colours, grey and pink. It then follows that this model has only two different stages associated with the set of situations corresponding to transfer events. Observe that it is straightforward to construct a bijection that implies the same probability distribution for $X(s_{14})$ and $X(s_{17})$. However these situations are at different stages since they have different colours. In this case, there is an implicit external condition that a valid bijection cannot be obtained by the permutation of labels the stage would require. This is a domain specific information that is embedded in this model. Of course, if a plausible sense could be discovered which would justify this association, then the external condition could be reframed.
	}

	\emph{
		We are able now to represent directly in the staged tree the context-specific statement corresponding to the transfer dynamic. Observe that by retaining the colour consistency between the process-driven objects we are able to analyse different branches of the process and compare them without having to draw the whole Staged Tree $\mathcal{ST}_\infty$. Also note that from these coloured trees we can see that the only difference between the two processes in time-slice~$1$ with respect to prisoners at situations $s_{14}$ and $s_{17}$ is in terms of the membership of their networks. This is because their corresponding situations are the only ones that have different colours (blue and brown).
	}

\end{myExampleCont2}

A staged tree can easily handle time-homogeneity condition enforced over the primitive probabilities in a sense formally defined below. For this purpose, let $\xi(s,k)$ be the time-ordered concatenation of events 
that precedes a situation $s$ in the time-invariant tree~$\mathcal{T}_{-1}$ and in the last $k$ time-slices. So, for example, in Figure~\ref{fig:ST_1_s15} (see also Figure~\ref{fig:ObjectOriented_InfiniteEventTree_Radicalisation_3var}) since there is no time-invariant event tree we have that $\xi(s_{59},0)=(Sporadic,Vulnerable)$ and  $\xi(s_{59},1)=(Sporadic,Vulnerable,No,Sporadic,Vulnerable)$.

\begin{myDef}
\label{def:TH-STT}
An $\boldsymbol{N}$ \textbf{Time-Homogeneous Staged Tree after time} $\boldsymbol{T}$ ({$N\text{THST-}T$}), $N \in \{0,\ldots,T\}$, is a $\mathcal{ST}_\infty$ whose corresponding event tree is  periodic after time~$T$ and for every situations $s_a$ and $s_b$, such that $s_a$ and $s_b$ are, respectively, situations in time-slice ${t_a = T, T+1,\ldots}$ and ${t_b = T,T+1,\ldots}$, we have that
\begin{equation}
\label{eq:Stationary_Staged_Tree}
\xi(s_a,N)=\xi(s_b,N) \Rightarrow \pi(s|s_a)=\pi(s|s_b).
\end{equation}
If $T$ is equal to $N$, then the staged tree is simply called an $N$~Time-Homoge-neous Staged Tree ($N$THST).  A Time-Homogeneous Staged Tree after time~$T$ (THST-$T$) is an $N$THST-$T$ for some $N, N \in \{0,\ldots,T\}$.
\end{myDef}

A THST-$T$ based on a TOG($\mathcal{T}_{-1},\mathcal{T}$) is an important type of staged tree because it can be implicitly constructed using a finite number of finite staged tree objects. Time-homogeneity and periodicity yielded by $\mathcal{T}$ imply that at the end of every time-slice ${t, t=T, T+1, \ldots}$, there is only one finite set of staged tree objects $\hat{\Gamma}_t$ that can unfold in the next time-slice $t+1$. The map ${h_t:\Upsilon(\mathcal{ST}_t) \to \hat{\Gamma}_t}$ that specifies which finite staged subtree unfolds from each leaf node of~$\Delta(\mathcal{ST}_t)$ has also a repeating structure determined by those two conditions. Formally, for all ${t, t=T,T+1,\ldots}$, there is a bijection ${\phi_t:\Upsilon(\mathcal{ST}_t)\to\Upsilon(\mathcal{ST}_T)}$ such as $h_t=h_T \circ \phi_t$. 

This implies that the number of stages is finite and so the number of~colours that is required to identify them in the THST-$T$ based on a TOG($\mathcal{T}_{-1},\mathcal{T}$). Of course, if an event tree has many different stages, modellers can then combine colours with different shape forms of vertices to depict the stages. Highly complex models may even require to replace colours by numbers for this purpose. However, because of its complexity a large event tree is often not a transparent means of communication and explanation.  In these cases, modellers and experts tend to focus on smaller subtrees using a hierarchical approach supported by event tree objects. We have therefore found that the coarse classification of parts of the tree using a limited number of colours to denote classes helps faithful communication with the expert.

Finally, we introduce a new concept that identifies situations whose unfolding processes in an infinite staged tree are probabilistically identical. This will enable us to embed the infinite staged tree into a more compact and easy-to-handle graphical model called Dynamic CEG using some simple graphical transformation rules.

\begin{myDef}
\label{def:position}
	Two situations $s_a$ and $s_b$ are said to be in the same \textbf{position}~$\boldsymbol{w}$ in an infinite staged tree if and only if besides the global and local conditions described in Definition~\ref{def:periodic_event_tree}, the bijection 
	\vspace{-7pt}
	$$\phi_w(s_a,s_b):\Lambda(\mathcal{ST}(s_a)) \to \Lambda(\mathcal{ST}(s_b)) \vspace{-7pt}$$
	satisfies the probabilistic condition that the ordered sequence of colours in a path $\lambda \in \Lambda(\mathcal{ST}(s_a))$ equals the ordered sequence of colours in the path ${\lambda'=\phi(s_a,s_b)(\lambda) \in \Lambda(\mathcal{ST}(s_b))}$.
	The partition yielded by positions over the set of situations in an staged tree $\mathcal{ST}_\infty$ is called \textbf{position structure}. 
\end{myDef}

Note that two processes developing from situations in the same position must be equivalent for the whole set of subsequent unfoldings along the staged tree. However processes evolving from situations in the same stage only have to be identified across the next step in their evolutions. Therefore, the set of positions constitutes a finer partition of the staged structure since situations in the same position are also in the same stage but the converse does not necessarily hold.

%%%%%%%%%%%%%%%%%%%%%%%%%%%%%%%%%%%%%%%%%%%%%%%%%%%%%%%%%%%%%%%
\section{A Dynamic Chain Event Graph}
\label{sec:DCEG_class}

A DCEG (Definition~\ref{def:dceg}) is a labelled directed multi-graph \cite{Diestel.2006} obtained by a graphical transformation of a staged tree. Every staged tree then spans a unique DCEG by vertex contraction operations over the set of situations. The main objective of this transformation is to encapsulate an infinite staged tree in a very compact graphical representation that facilitates the interpretation and reasoning of the corresponding model.

\begin{myDef}
	\label{def:dceg}
	A \textbf{Dynamic Chain Event Graph} (DCEG) is a directed graph obtained from a staged tree by merging all situations in the same position into a single vertex and then gathering, if they exist, all leaf vertices into a single position~$w_\infty$. A \textbf{DCEG model} associated with a DCEG~$\mathbb{C}$ is a graphical model whose sample space is represented by the supporting event tree of $\mathbb{C}$ and whose probability measure respects the set of conditional independence statements depicted by~$\mathbb{C}$.  
\end{myDef}

Observe that in a DCEG each vertex corresponds to a position and so retains the colour associated with its corresponding stage in the staged tree. In many instances, for the sake of economy of colours and clarity of the DCEG graph, it may be agreed that stages with a single position are all showed without colour; otherwise, it keeps the colour of its stage in the staged tree. For an example of this colour simplification in CEGs, see~\cite{Collazo.Smith.2016}. 

As observed in \cite{Barclay.etal.2015} a DCEG may have directed loops that allow it to have a finite number of vertices. According to Theorem~\ref{the:periodicity_finite_DCEG}, an event tree needs to be periodic in order to support a finite DCEG. Of course, this condition is not sufficient. Otherwise, the mapping of an infinite staged tree into a finite DCEG would be independent of the probability measure corresponding to the model. Theorem~\ref{the:finite_DCEG} below tells us that a necessary and sufficient condition is that there exists a time-slice~$T$ such that every situation whose parent is at time-slice~$T\!-\!1$ is in the same position as a situation that happened in any previous time-slice~$t$, $t=0,\ldots,T\!-\!1$. This can be a quite difficult condition to verify in practice given a staged tree. However, Theorem~\ref{the:THST_finite_DCEG} asserts that time-homogeneity suffices to satisfy this condition. 

\begin{myTheorem}
	\label{the:periodicity_finite_DCEG}
	If a DCEG is finite, then its supporting event tree is periodic after some time~T. 
\end{myTheorem}
\begin{proof}
	See \ref{app:theorem_periodicity_finite_DCEG}.
\end{proof}

\begin{myTheorem}
\label{the:finite_DCEG}
A DCEG supported by a staged tree $\mathcal{ST}_{\!\infty}$ is finite if and only if, for some time-slice~$t_b$, every situation~$s_b$ in~$l(\mathcal{ST}_{\!t_b})$ is in the same position as a situation~$s_a$ in~$\mathcal{ST}_{\!t_a}$, $t_a=0,\ldots,{t_b}$. 
\end{myTheorem}
\begin{proof}
See \ref{app:theorem_finite_DCEG}.
\end{proof}

\begin{myTheorem}
\label{the:THST_finite_DCEG}
Every time-homogeneous staged tree after time $T$ has an associated DCEG with a finite graph.
\end{myTheorem}
\begin{proof}
See \ref{app:THST_finite_DCEG}.
\end{proof}

As pointed out in Section~\ref{sec:Introduction}, the definition of a DCEG in~\cite{Barclay.etal.2015} may lead us to some topological inconsistencies when transferred to a Markov setting like the one we describe here. This type of problems arises because the direct extension of the concept of position from a CEG to a DCEG does not enforce a bijective map between the infinite staged tree and its corresponding DCEG graph. Therefore, the definition of position in~\cite{Barclay.etal.2015} does not always preserve the time-slice structure that is essential to the type of processes we analyse in this work. This is important because to obtain a finite DCEG graph it is necessary to propose a graphical semantic that enforces a loop over time-slices and not within a time-slice. As we are working with an infinite tree if we do not add some further structure to the idea of position we can end with a loop within a time-slice. In this case, the readers of a DCEG model cannot determine when a unit will get out of a loop. This additional structure is represented here by the local condition of the bijection in Equation~\ref{eq:bijection_periodicity_tree} which enforces a correspondence of ordered sequences of events and colours locally (within each time-slice) and not only globally (in the whole path) as in~\cite{Barclay.etal.2015}. For further discussion through an example, see~\cite{Collazo.2017, Collazo.Smith.2017}. 

\subsection{An $N$ Time-Slice Dynamic Chain Event Graph}
\label{subsec:NT-DCEG}

Definition \ref{def:T_position} below introduces the concept of a $T$-position. This demands a further constraint on the definition of a regular position. So situations in the same $T$-position are always in the same position but the converse is not necessarily valid. 

\begin{myDef}
\label{def:T_position}
Two situations $s_a(t_a)$ and $s_b(t_b)$ are in the same $\boldsymbol{T}$\textbf{-position} if and only if they are in the same position, and one of the following conditions hold: $t_a,t_b \in \{T,T+1,\ldots\}$ or $t_a=t_b=t$, $t \in\{-1,0,\ldots,T-1\}$. They are said to be in the same $\boldsymbol{\infty}$\textbf{-position} if and only if they are in the same position and happens at the same time-slice~$(t_a=t_b)$.
\end{myDef}

A T-position avoids cycles before a time-slice $T$ whilst preserving all other characteristics of a standard DCEG. Using this construction we can demand that a finite DCEG has all its loops rooted at situations that happen at the same time-slice if its staged tree is time-homogeneous after some time $T$. Based on $T$-positions, we can now define a useful DCEG class, called the $N$~Time-Slice Dynamic Chain Event Graph ($N$T-DCEG).

\begin{myDef}
\label{def:NTDCEG}
Take a time-homogeneous staged tree after time $N\!-\!1$ whose supporting event tree can be elicited as a {$\text{TOG}(\mathcal{T}_{-1},\mathcal{T}$)}.
An $\boldsymbol{N}$~\textbf{Time-Slice Dynamic Chain Event Graph} (\text{$N$T-DCEG}), $N=2,3,\ldots$, is a directed coloured graph obtained from  this type of event tree by:
	\begin{enumerate}[nosep]
		\item merging all situations in the same $(N\!-\!1)$-position into a single vertex;
		\item diverting all leaf vertices at time-slice~$t$, ${t=-1,0,\ldots,N-2}$, if they exist, to a single sink vertex~$w_\infty^t$; and
		\item gathering all leaf vertices from time-slice~$N\!-\!1$ on, if they exist,  into a single sink vertex~$w_\infty$.
	\end{enumerate}
An $\boldsymbol{N}$\textbf{T-DCEG model} corresponding to an $N$T-DCEG~$\mathbb{C}$ is a graphical model whose sample space is represented by the supporting event tree of~$\mathbb{C}$ and whose probability measure respects the set of conditional independence statements depicted by~$\mathbb{C}$. Henceforth we will denote the set of all sink vertices in an \text{$N$T-DCEG} graph as $\mathcal{W}_{\!\infty}$. We will also define $\mathcal{W}_{\!\infty}^T$, ${T=-1,0,\ldots,N\!-\!2}$, as the set of all sink vertices associated with time-slices~$t$, $t \in \{-1,0,\ldots,T\}$.
\end{myDef}

An $N$T-DCEG has a unique periodic graphical structure over all time-slices and its primitive probabilities are all time-homogeneous for time-slices~$t$, $t = N\!-\!1,N,\ldots$.  Theorem~\ref{the:THST_finite_DCEG} guarantees that an $N$T-DCEG is also a finite graph. Note that in many real-world applications a time-slice~$T$ might exist with the property that it is possible to obtain the same graphical model regardless of whether the nodes of the graph represented a position or a $T\text{-position}$. In Example \ref{ex:Radicalisation_Dynamic_3variables} this is actually the case if we adopt $T\!=\!0$ or $T\!=\!1$. The standard DCEG and a 2T-DCEG (Figure~\ref{fig:2T-DCEG_Radicalisation_3var}) that each represent the radicalisation process will then be identical. Note that to draw an uncluttered graph without any loss some of its edges are dashed and grey.

\vspace{0pt}
\begin{figure}[t]
	\begin{center}
		\includegraphics[scale=0.475,angle=-90,origin=c,trim=0 0 0 -150]{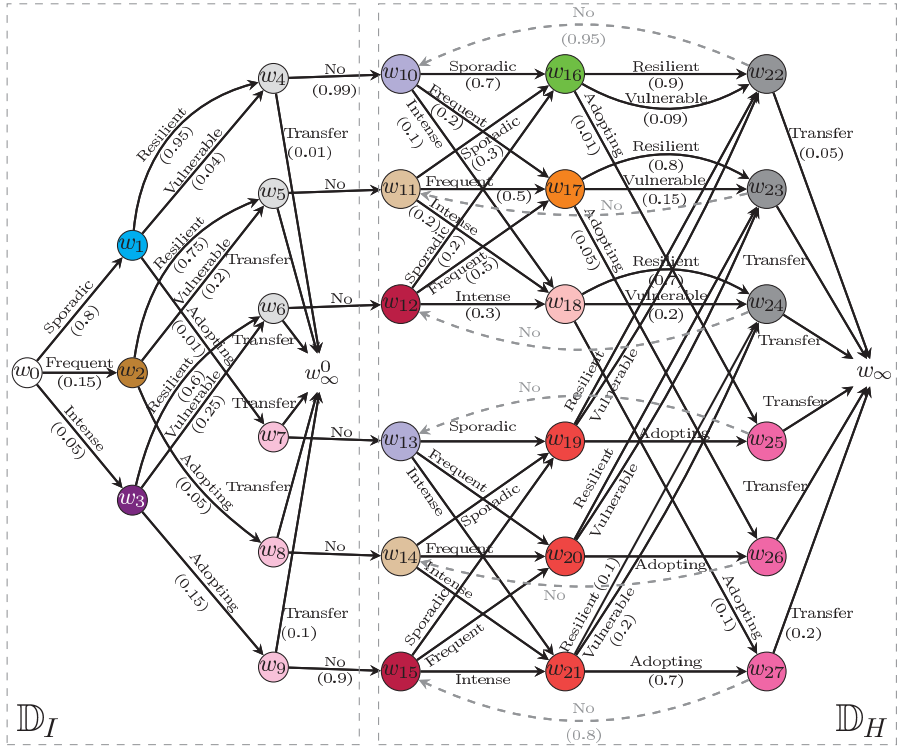}
	\end{center}
	\vspace{-27pt}
	\caption{The 2T-DCEG associated with Example \ref{ex:Radicalisation_Dynamic_3variables}. The stage structure is given by the following partition: $u_0=\{w_0\}$, $u_1=\{w_1\}$, $u_2=\{w_2\}$, $u_3=\{w_3\}$, $u_4=\{w_4,w_5,w_6\}$, $u_5\!\!=\!\!\{w_7,w_8,w_9\}$, $u_6\!\!=\!\!\{w_{10},w_{13}\}$, $u_7\!\!=\!\!\{w_{11},w_{14}\}$, $u_8\!=\!\{w_{12},w_{15}\}$, $u_9\!=\!\{w_{16}\}$, ${u_{10}\!=\!\{w_{17}\}}$, $u_{11}=\{w_{18}\}$, $u_{12}=\{w_{19},w_{20},w_{21}\}$,  $u_{13}=\{w_{22},w_{23},w_{24}\}$, $u_{14}=\{w_{25},w_{26},w_{27}\}$. The conditional probability of an event associated with a particular stage is shown in parentheses. \label{fig:2T-DCEG_Radicalisation_3var}}
	\vspace{-7pt}
\end{figure}

Recall that an event tree $\mathcal{T}_{-1}$ is associated with time-invariant information and a finite event tree $\mathcal{T}$ fully characterises every time-slice of a model obtained from a {\text{TOG}$(\mathcal{T}_{-1},\mathcal{T}$) (see Section \ref{subsec:Event_Tree}). Therefore, an ${N\text{T-DCEG}}$ requires us to elicit only two finite process-driven objects: $\mathcal{T}_{-1}$ and $\mathcal{T}$. To obtain a staged tree, because of the time-homogeneity condition it is necessary to define explicitly only those primitive probabilities associated with the first $N$ time-slices.
	
From a graphical point of view the use of an $N\!-\!1$-position structure enables us to enforce loops only from time-slice $N\!-\!1$ on. This is important when we need to merge DCEGs that are spanned by different branches of the same event tree. For example, based on an event-tree that splits the process according to some time-invariant attributes a distributed model construction can compose the domain information coherently. In this case, every subprocess~i has its own particular $N_i$T-DCEG model. To merge these $N_i$T-DCEGs and so to stress common periodic characteristics that may be shared between them, it is helpful to demand that all loops must be rooted at situations that happen at the same time-slice~$N$, where $N=\max_i N_i$. Since this condition does not constrain the model space, it is not strictly necessary. However, restricting the graphical depiction of the statistical model does facilitate the readability of the final $N$T-DCEG and the design of efficient algorithmic structures. For further discussion using an example, see \cite{Collazo.Smith.2018,Collazo.2017}.

It is useful at this stage to introduce the definition of temporal edge in an $N$T-DCEG~$\mathbb{C}$. Note that a temporal edge associated with time $N\!-\!1$ will also be a temporal edge associated with every time~$t, t=N,N\!+\!1,\ldots$. This happens because of the time-homogeneity condition required from every $N$T-DCEG. These temporal edges $-$~called cyclical temporal edges below~$-$ enable us to represent a time-homogeneous map that connects positions in two consecutive time-slices.  

\begin{myDef}
\label{def:temporal_edges}
Take an ${N\text{T-DCEG}}$ $\mathbb{C}$ obtained from an infinite tree $\mathcal{T}_\infty$. A directed edge $(w_a,w_b)$ in $\mathbb{C}$ is a \textbf{temporal edge} associated with time-slice~$t$ if and only if there exist two situations $s_a \in w_a$ and $s_b \in w_b$ such that $s_b \in ch(s_a)$, $ch(s_b)\neq\emptyset$ and $s_b \in l(\mathcal{T}_t)$, where $\mathcal{T}_t \subset \mathcal{T}_\infty$. We call a temporal edge associated with time-slices~$t$, $t=N\!-\!1,N,\ldots$, of $\mathbb{C}$ a \textbf{cyclical temporal edge}. Henceforth we will denote the set of cyclical temporal edges by $E_\text{\circled{$\dagger$}}$. We will also define $\mathcal{W}_{Head}$ and $\mathcal{W}_{Tail}$ as the sets of positions of $\mathbb{C}$ that are, respectively, the heads and tails of cyclical temporal edges.
\end{myDef}

These concepts enable us to demonstrate a close connection between ${N\text{T-DCEG}}$~models and Markov Chains. For this purpose, take the state space $\mathcal{X}=\mathcal{W}_{Head}$, if $\mathcal{W}_{\!\infty} = \emptyset$, or $\mathcal{X}=\mathcal{W}_{Head}  \cup \{w_\infty\}$, otherwise. Let $\boldsymbol{\mu}\!=\!(\mu_1,\!\ldots\!,\mu_{|\mathcal{X}|})$ be an initial distribution, where $\mu_i, i=1,\ldots,|\mathcal{X}|$, is the probability of a position~${w_{k_i}}$ in~${\mathcal{W}_{Head}}$ to be reached at the end of $N\!-\!1$ time-slices. In other words, each $\mu_i$, such that $k_i \neq \infty$, is equal to the sum of the occurrence probabilities associated with each ${w_0\text{-to-}w_{k_i}}$ path in an \text{$N$T-DCEG}. This can then be translated into the sum of occurrence probabilities associated with each root-to-$s_j(N\!-\!1), s_j(N-1) \in w_{k_i}$, path in the event tree. We therefore have that
\vspace{-7pt}
\begin{equation}
\label{eq:MC_initial_distribution}
\mu_i
=
	\sum_{s_a(N\!-\!1) \in w_{k_i}}P(\lambda(s_0,s_a))
=
	\sum_{s_a(N\!-\!1) \in w_{k_i}}\prod_{s \in \Psi(s_a)} \pi(\psi(s,s_a)|s),
\vspace{-7pt}
\end{equation}
where $\lambda(s_a,s_b)$ denotes the $s_a$-to-$s_b$ path in the event tree. 

By convention $\mu_1$ is always associated with the position~$w_\infty$, if $w_\infty \in \mathcal{X}$. In this case and if $\mathcal{W}_{\!\infty}^{N\!-\!2} = \emptyset$, then set $\mu_1=0$ since no unit reaches a sink position during the first~$N\!-\!1$ time-slices. Otherwise, $u_1$ is the probability of a unit arriving at a sink position $w_\infty^t$ in $\mathcal{W}_{\!\infty}^{N\!-\!2}$. This is then given by
\vspace{-7pt}
\begin{equation}
\label{eq:MC_initial_distribution_u1}
\mu_1
=
\sum_{\substack{l_a \in w_\infty^t \\ w_\infty^t \in \mathcal{W}_{\!\infty}}}
		P(\lambda(s_0,l_a))
=
\sum_{\substack{l_a \in w_\infty^t \\ w_\infty^t \in \mathcal{W}_{\!\infty}}}
		\prod_{s \in \Psi(l_a)} 
				\pi(\psi(s,l_a)|s).
\vspace{-7pt}
\end{equation} 

Now define $\boldsymbol{M}=[m_{ij}]$ as a transition matrix, where $m_{ij}$ represents the transition probability from a state $w_{k_i}$, $w_{k_i} \in \mathcal{X}$,  to a state $w_{k_j}$, $w_{k_j} \in \mathcal{X}$. Each $m_{ij}$ corresponds to the sum of the probabilities associated with each walk that goes from a position $w_{k_i}$, $k_i \neq \infty$, to a position $w_{k_j}$ in only one time-slice. If $w_{k_j}$ cannot be reached from $w_{k_i}$ in one time-slice,  then $m_{ij}=0$. Also fix $m_{11}=1$, if $k_1=\infty$. Again, every non-null $m_{ij}$, such that $k_i \neq \infty$, can be expressed as the following function of primitive probabilities:
\vspace{-7pt}
\begin{equation}
\label{eq:MC_transition_matrix}
m_{ij}
=
	\!\!\!
	\sum_{s_a(N\!-\!1) \in w_{k_i}} 
				\sum_{v_b(N) \in w_{k_j}} 
						\!\!\!\!\!\! P(\lambda(s_a,v_b))
=
	\!\!\!
	\sum_{s_a(N\!-\!1) \in w_{k_i}} 
				\sum_{v_b(N) \in w_{k_j}}
						\prod_{s \in \Psi(s_a,v_b)} 
								\!\!\!\!\!\! \pi(\psi(s,v_b)|s).
%\vspace{-7pt}
\end{equation}    

Theorem \ref{theo:NTDCEG_MC} below tells us that every $N$T-DCEG can be interpreted as a Markov Chain with a finite state transition diagram defined by a state space $\mathcal{X}$, an initial distribution~$\boldsymbol{\mu}$ and a transition matrix~$\boldsymbol{M}$.  

\begin{myTheorem}
\label{theo:NTDCEG_MC}
There is a map from every NT-DCEG into a finite state-transition diagram.
\end{myTheorem}
\begin{proof}
Construct a Markov Chain whose state space is given by $\mathcal{X}$ as defined above. Take the initial distribution as given by Equations~\ref{eq:MC_initial_distribution} and~\ref{eq:MC_initial_distribution_u1} and the transition matrix as given by Equation~\ref{eq:MC_transition_matrix}, respectively. The state-transition diagram of this Markov Chain is then finite. So the result follows.
\end{proof}

This is an important link enabling the corresponding $N$T-DCEG to be represented in a very compact way. In many real-world problems, it may be challenging to directly identify the states of a Markov Chain and to elicit the whole process when domain experts only observe sequence of events. Based on an $N$T-DCEG model we can now construct and learn a Markov Chain that may be synergistically used to gain a deep understand of a dynamic process. 

Focusing only on the transitions between time-slices, the Markov Chain projection provides a framework for domain experts to analyse how the system may develop over time. For example, domain experts can explore the equilibrium state of the Markov Chain and can also obtain the respective rate of convergence to it given the actual state of the process. These analytical results may suggest the necessity of some systemic intervention. In order to perform such an exploration it can be helpful to return to the $N$T-DCEG model and zoom in again over the conditional independences depicted into its corresponding $N$T-DCEG graph. 

Corollary~\ref{cor:ergodic_irreducible_MC} guarantees that a Markov Chain spanned by an $2$T-DCEG based on a {$\text{TOG}(\emptyset,\mathcal{T}$)} Type~A all of whose primitive probabilities are strictly positive has a unique stationary distribution. In contrast, an $2$T-DCEG yielded by {$\text{TOG}(\mathcal{T}_{-1},\mathcal{T}$)} Type~B always has at least one absorbing state because $w_\infty \in \mathcal{X}$.  

\begin{myCorollary}
	\label{cor:ergodic_irreducible_MC}
	Every $2$T-DCEG~$\mathbb{C}$ obtained from a {$\text{TOG}(\mathcal{T}_{-1},\mathcal{T}$)} Type A whose probability associated with each edge is non-null and $\mathcal{T}_{-1}=\emptyset$ has a corresponding Markov process that is ergodic and irreducible.
\end{myCorollary}
\begin{proof}
	The periodicity yielded by~$\mathcal{T}$ and the time-homogeneity condition guarantee that there is a path in the underlyning staged tree of~$\mathbb{C}$ that leads from every position $w_i \in \mathcal{W}_{Head}$ to any position in $\mathcal{W}_{Head}$ in only one time-slice. Since the primitive probabilities are all positive, every unit in a position $w_i \in \mathcal{W}_{Head}$ has a non-zero probability of returning to the same position $w_i$ or to reach a position $\mathcal{W}_{Head} \backslash \{w_i\}$ in the end of a time-slice. From the definition of the Markov Chain described in the proof of Theorem~\ref{theo:NTDCEG_MC}, it can then be seen that the corresponding Markov Chain is ergodic and irreducible.
\end{proof}

The association between $N$T-DCEG models and Markov Chains is further discussed in the example below.

\begin{myExampleCont1}
\emph{
	Figure \ref{fig:MC_Radicalisation_3var} depicts the state-transition diagram of the Markov Chain corresponding to the ${N\!\text{T-DCEG}}$ showed in Figure \ref{fig:2T-DCEG_Radicalisation_3var}. Using Equations~\ref{eq:MC_initial_distribution} and~\ref{eq:MC_initial_distribution_u1}, we can calculate the initial distribution 
	\vspace{-9pt}
	$${\boldsymbol{\mu}=[0.012,0.784,0.141,0.042,0.007,0.007,0.007]}. \vspace{-23pt}$$
	}

\emph{
	The transition matrix is obtained from Equation~\ref{eq:MC_transition_matrix} and is given by
	}

\vspace{-19pt}
\begin{equation}
\label{eq:transition_matrix_example}
\boldsymbol{M}=
\begin{bmatrix}
& w_{\infty} & w_{10} & w_{11} & w_{12} & w_{13} & w_{14} & w_{15} \\
w_{\infty} & 1.000 & 0.000 & 0.000 & 0.000 & 0.000 & 0.000 & 0.000 \\
w_{10} & 0.057 & 0.657 & 0.181 & 0.067 & 0.006 & 0.008 & 0.024 \\
w_{11} & 0.063 & 0.283 & 0.451 & 0.133 & 0.002 & 0.020 & 0.048 \\
w_{12} & 0.068 & 0.187 & 0.451 & 0.002 & 0.002 & 0.020 & 0.072 \\
w_{13} & 0.154 & 0.200 & 0.057 & 0.029 & 0.392 & 0.112 & 0.056 \\
w_{14} & 0.154 & 0.086 & 0.143 & 0.057 & 0.168 & 0.280 & 0.112 \\    
w_{15} & 0.154 & 0.057 & 0.143 & 0.086 & 0.112 & 0.280 & 0.168
\end{bmatrix}
\vspace{-7pt}
\end{equation}
 
\emph{
	The Markov Chain enables us to present the radicalisation process compactly using just a few positions of the elicited ${N\!\text{T-DCEG}}$. Being based on a tree, an ${N\!\text{T-DCEG}}$ provides domain experts with an intuitive framework not only to represent and estimate a process but also to interpret it using a single random variable whose states over time are given by a finite set of positions. }

%\vspace{-15pt}
\begin{figure}[t]
	\begin{center}
		\includegraphics[scale=0.35,angle=-90,origin=c,trim=0 0 0 -130]{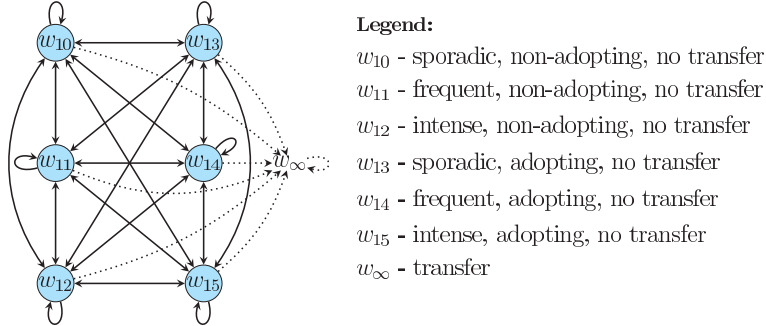}
	\end{center}
	\vspace{-83pt}
	\caption{The state-transition diagram associated with the 2T-DCEG depicted in Figure~\ref{fig:2T-DCEG_Radicalisation_3var} \label{fig:MC_Radicalisation_3var}}
	\vspace{-25pt}
\end{figure}

\emph{
	Here the radicalisation process can be explained using a random variable that has seven states represented by the positions $w_{10},\ldots,w_{15}$ and $w_\infty$. Note that the state-transition diagram tells us that the categories Resilient and Vulnerable can be merged without losing any useful information for the macro-level interpretation of the chosen ${N\!\text{T-DCEG}}$. It also follows directly that all states connected with prison transfers can be represented by only one absorbing state $w_\infty$. In this way, the state-transition diagram provides us with an evocative picture of the overall dynamic development over time. Furthermore, its transition matrix can be obtained by learning the ${N\!\text{T-DCEG}}$, for example, through conjugate Bayes learning~\cite{Barclay.etal.2015}.}

\end{myExampleCont1}

%%%%%%%%%%%%%%%%%%%%%%%%%%%%%%%%%%%%%%%%%%%%%%%%%%%%%%%%%%%%%%%
\section{The Relationship between an $N$T-DCEG and a CEG}
\label{sec:CEG_NTDCEG}

Take a DCEG $\mathbb{C}$ based on a staged tree $\mathcal{ST}_\infty$. For every time-slice~$t,$ ${t=0,1,\ldots}$, construct a CEG $\mathbb{C}_t$ spanned by the staged tree~$\mathcal{ST}\!_t$, ${\mathcal{ST}\!_t \subset \mathcal{ST}_\infty}$, using the concept of $\infty$-position. Then for every $t, t=0,1,\ldots$, the set of primitive probabilities 
\vspace{-11pt}
$$\Pi_t=\{\pi(v|s_i);v \in ch(s_i),s_i \in \mathcal{T}_t \subset \mathcal{T}_\infty\} \vspace{-5pt}$$
defines a consistent probability measure over the path $\sigma$-algebra of  $\mathbb{C}_t$ (see~\cite{Smith.Anderson.2008}), where $\Pi_t$ is a subset of~$\Pi$ and $\Pi$ is the set of primitive probabilities of $\mathbb{C}$. The path $\sigma\text{-algebras }$ $\mathcal{F}_t=\sigma\{\Lambda(v_i);v_i\in \mathcal{T}_t\}$ associated with each CEG $\mathbb{C}_t$ constitute a natural filtration of the path-cylinder $\sigma$-algebra $\mathcal{F}=\sigma\{\Lambda(v_i);v_i\in \mathcal{T}_\infty\}$ corresponding to $\mathbb{C}$.
The DCEG probability space can then be equipped with a useful set of CEGs $\mathcal{F}(\mathbb{C})=\{\mathbb{C}_t;t=0,1,\ldots\}$. 

Having the same stage structure, both a DCEG $\mathbb{C}$ and a CEG $\mathbb{C}_T$ depict equivalent conditional independences if the interest lies in the 1-step unfolding of events that may happen from a specific situation at time-slice~~$t$, ${t=0,\ldots,T}$. However, this fact does not hold for analyses that involve a development over two or more steps.
This is because positions in a CEG are defined using \emph{finite} subtrees expressing only the early unfoldings of the process whilst positions in a DCEG are based on \emph{infinite} subtrees. Therefore all situations at time~$t$, ${t=0,\ldots,T}$, merged into a single position in~$\mathbb{C}_T$ will not necessarily be collected by a unique equivalent position in~$\mathbb{C}$.

%\vspace{-20pt}
\begin{figure}[t] 
	\begin{center}
		\includegraphics[scale=1.027,angle=-90,origin=c,trim=5 43 -50 -39]{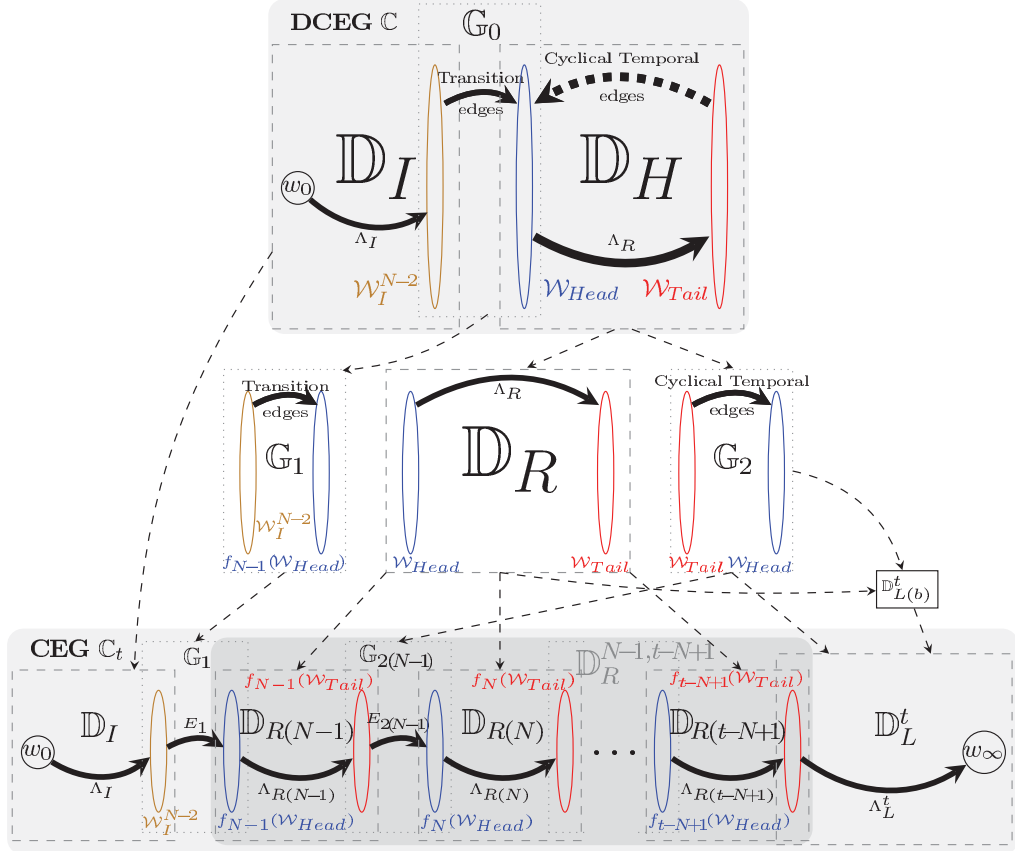}
	\end{center}
	\vspace{-17pt}
	\caption{The process of obtaining a CEG $\mathbb{C}_t, t=2N\!-\!2,2N\!-\!1,\ldots$, from a DCEG~$\mathbb{C}$ supported by a TOG($\emptyset,\mathcal{T}$) type~$A$. Schematic representation of Theorem \ref{theo:C_k}.  \label{fig:NT-DCEG_as_CEG}}
	\vspace{-5pt}
\end{figure}

Fortunately there is a stronger link between an $N\text{T-DCEG}$~$\mathbb{C}=(V,E)$ and a CEG~$\mathbb{C}_t$ in~$\mathcal{F}(\mathbb{C})$. It allows us to express every $\mathbb{C}_t, t\!=\!N\!-\!1,N,\ldots$, using subgraphs of $\mathbb{C}$. This result expressed in Theorem~\ref{theo:C_k} is very important because it enables us to present a methodology at the end of this section to construct an $N$T-DCEG using the tree objects. It also supports the development of methods to read conditional independences in the next section. 

To obtain Theorem~\ref{theo:C_k} we first need to identify these subgraphs and to explain how they can be extracted from~$\mathbb{C}$. Before rigorously introducing this construction, Figure~\ref{fig:NT-DCEG_as_CEG} depicts schematically how we do this. Observe that every $N$T-DCEG~$\mathbb{C}$ has two important subgraphs, $\mathbb{D}_I$ and $\mathbb{D}_H$. The subgraph $\mathbb{D}_I$ initialises the modelled process over the first $N\!\!-\!\!1$ time-slices. The cyclic subgraph $\mathbb{D}_{H}$ represents the time-homogeneous developments of the process and then contains the cyclical temporal edges from time-slice $t$ to $t+1, {t=N\!\!-\!\!1,N,\ldots}$. These two subgraphs are connected by a bipartite graph $\mathbb{G}_0$ whose temporal edges we will call transition edges. Intuitively, Theorem~\ref{theo:C_k} enables us to unfold the infinite time-homogeneous time-slices summarised in $\mathbb{D}_H$ into a finite CEG $\mathbb{C}_t$ using the graphs $\mathbb{D}_R$ and $\mathbb{G}_2$.

To formally introduce these subgraphs, let~$\mathcal{W}_I^t$ and~$\mathcal{W}_J^t$ , ${t=0,\ldots,N\!-\!2}$, be the sets of all positions that are, respectively, the tails and heads of temporal edges associated with time-slice~$t$. Note that $\mathcal{W}_J^{N\!-\!2}=\mathcal{W}_{Head}$. Also let ${\mathcal{W}_{I(\infty)}^t=\mathcal{W}_I^t \cup \mathcal{W}_{\!\infty}^t}$, 
$\mathcal{W}_{I(\infty)} = \mathcal{W}_{Tail} \cup \mathcal{W}_\infty \backslash \mathcal{W}_\infty^{N\!-\!2}$ and $E^- = E \backslash E_\text{\circled{$\dagger$}}$.
Take any directed graph ${\mathbb{G}=(V_\mathbb{G},E_\mathbb{G})}$ and denote by~$a(V_a)$, ${V_a \subseteq V_\mathbb{G}}$, the set of all vertices in~$\mathbb{G}$ that are antecedents of at least one vertex in $V_a$. 
Finally, for every subset of positions~$\mathcal{W}$ of $V$ define the bijective label transformations ${f_t: \mathcal{W} \to \mathcal{W}^t}$, ${t=0,1,\ldots}$, such that: $f_0(w_i)=w_i$;  ${f_t(w_i)=w_i^t}$, for all position ${w_i \in \mathcal{W} \backslash \mathcal{W}_{\!\infty}}$ and $t=1,2,\ldots$; $f_t(w_\infty^T)=w_\infty^T$, if $w_\infty^T \in \mathcal{W}$; and $f_t(w_\infty)=w_\infty$, if $w_\infty \in \mathcal{W}$. We begin defining the graph~$\mathbb{D}_I$ and two associated families of graphs, which are used to construct CEGs based on the first $N\!-\!1$ time-slices at most. 

\begin{myDef}
	\label{def:initial_graph}
	The \textbf{initial graph}~${\mathbb{D}_I}=(V_I,E_I)$ corresponds to the set~$\Lambda_I$ of~$w_0\text{-to-}\mathcal{W}_{I(\infty)}^{N\!-\!2}$~paths in~$\mathbb{C}$, such that
	$V_I=
		\{w \in \mathbb{C};
			w \in a(\mathcal{W}_{I(\infty)}^{N\!-\!2}) 
				\cup \mathcal{W}_{I(\infty)}^{N\!-\!2}\}$
	and 
	$E_I=
		\{e(w_i,w_j)\in E^-;
				w_i,w_j \in V_I\}$.
\end{myDef}

\begin{myDef}
	\label{def:t_initial_graph}
	The \textbf{$t$-initial graph}~$\mathbb{D}_{I(*)}^t=(V_{I(*)}^t,E_{I(*)}^t)$, $t=0,\ldots,N\!-\!2$, is defined by the vertex set $V_{I(*)}^t=\{w \in\ \mathbb{C};w \in a(\mathcal{W}_{I(\infty)}^t) \cup \mathcal{W}_{I(\infty)}^t \cup \mathcal{W}_J^t\}$ 
	and the edge set
	${E_{\!I(*)}^t\!=\!\{e(w_i,w_j) \!\in\! E^-;
				w_i \!\in\! a(\mathcal{W}_{I(\infty)}^t) \cup \mathcal{W}_{I}^t, w_j \!\in\! V_{I(*)}^t\}}$.
\end{myDef}

\begin{myDef}
	\label{def:transformed_t_initial_graph}
	The \textbf{transformed $t$-initial graph} $\mathbb{D}_{I(\infty)}^t=(V_{I(\infty)}^t,E_{I(\infty)}^t)$, $t=0,\ldots,N\!-\!2$, is obtained from $\mathbb{D}_{I(*)}^t$ by merging its vertices without children into a single node~$w_\infty$, such that
	${V_{\!I(\infty)}^t \!\!=\!
		\{w \in \mathbb{C};w \!\in\! a(\mathcal{W}_{I(\infty)}^t) \cup \mathcal{W}_I^t \} \cup \{w_\infty\}}$,
	$E_{\!I(\infty)}^t \!\!=\! E_{I(\infty)}^{t(a)} \cup E_{I(\infty)}^{t(b)}$,
	${E_{\!I(\infty)}^{t(a)} =
		\{e(w_i,w_j) \in E^- ;w_i,w_j \in V_{I(\infty)}^t \setminus \{w_\infty\}\}}$~and
	${E_{\!I(\infty)}^{t(b)} \!=\!
		\{e(w_i,w_\infty); e(w_i,w_j) \!\in\! E^-, w_i \in V_{\!I(\infty)}^t \!\!\!\setminus\!\! \{w_\infty\}, w_j \in \mathcal{W}_{\infty}^t \cup \mathcal{W}_J^t}\}$.
\end{myDef}

The isomorphic graphs $\mathbb{G}_r, r=0,1$, provide the link between the time-slices $N-2$ and $N-1$. They then connect the initial graph~$\mathbb{D}_I$ and the cyclic graph~$\mathbb{D}_H$.  

\begin{myDef}
	\label{def:transition_initial_graph}
	The \textbf{initial $r$-link graph} $\mathbb{G}_r=(V_r,E_{r}), r=0,1$, corresponds to the vertex set
	$V_{r}=\mathcal{W}_I^{N\!-\!2} \cup f_{(N-1)*r}(\mathcal{W}_{Head})$ and the edge set
	$E_r=\{e(w_i,f_{(N-1)*r}(w_j)); e(w_i,w_j) \in E^-,{ w_i \in \mathcal{W}_I^{N\!-\!2}}, {w_j \in \mathcal{W}_{Head}} \}$.
\end{myDef}

We now formally present the cyclic graph $\mathbb{D}_H$ and some families of derived graphs. This enables us to add the time-homogeneous structure to CEGs extending over $N-1$ time-slices.

\begin{myDef} 
	\label{def:cyclic_graph}
	The \textbf{time-homogeneous} $\mathbb{D}_H=(V_H,E_H)$ is a cyclic subgraph of~$\mathbb{C}$, such that $V_H=V \backslash V_I$ and ${E_H\!=\!\{e(w_i,w_j) \in E \backslash (E_I \cup E_0)\}}$.
\end{myDef}

\begin{myDef}
	\label{def:repeating_graph}
	The \textbf{repeating graph} ${\mathbb{D}_R=(V_R,E_R)}$, where $V_r=V_H$ and $E_r=E_H \backslash E_\text{\circled{$\dagger$}}$, is an acyclic subgraph of $\mathbb{D}_H$ that is made up of the set~$\Lambda_R$ of paths that unfold from a position in $\mathcal{W}_{Head}$ and arrive at a position at~$\mathcal{W}_{I(\infty)}$.
\end{myDef}

\begin{myDef}
	\label{def:t_reepating_graph}
	The \textbf{$t$-repeating graph} $\mathbb{D}_{\!R(t)}\!=\!(V_{\!R(t)},E_{\!R(t)})$, ${t\!=\!N\!\!-\!\!1,N,\ldots}$, is obtained from $\mathbb{D}_R$ when its vertices are relabelled by the transformation~$f_t$ as follows:
	$\!V_{\!R(t)}\!\!=\!\!\{f_t(w);w \!\in\! V_R \}$ and ${E_{\!R(t)}\!\!=\!\!\{e(f_t(w_a),f_t(w_b));e(\!w_a,w_b) \!\in\! E_R\}}$. Let $\Lambda_{R(t)}$ be the set of $f_t(\mathcal{W}_{Head})\text{-to-}f_t(\mathcal{W}_{I(\infty)})$ paths in~$\mathbb{D}_{\!R(t)}$.
\end{myDef}

The graph $\mathbb{D}_{R(t)}$ is isomorphic to $\mathbb{D}_{R}$ and represents $\mathbb{D}_{R}$ at time $t$. To construct the graph~$\mathbb{D}_R^{t_a,t_b}$ associated with the time-homegenous process between time-slices~$t_a$ and~$t_b$, $t_b>t_a\geq N\!-\!1$, it is necessary to connect together the graphs $D_{R(t)}$, $t=t_a,t_a\!+\!1,\ldots,t_b$. For this propose, we define the graphs~$\mathbb{G}_{2(t)}$ and a particular union operation between two graphs below. The graphs~$\mathbb{G}_{2(t)}$ are isomorphic to the graph~$\mathbb{G}_2$ constituted by the temporal edges of~$\mathbb{C}$. Since the edge set $E_{2(t)}$ is spanned by the set of cyclical temporal edges of $\mathbb{C}$, the graph $\mathbb{G}_{2(t)}$ then represents the dependence structure between time-slices $t$ and $t+1$, ${t\!=\!N\!\!-\!\!1,N,\ldots}$.

\begin{myDef}
	\label{def:time_homogeneous_link_graph}
	The \textbf{time-homogeneous link graph} ${\mathbb{G}_{2}=(V_{2},E_{2})}$ is a subgraph of~$\mathbb{C}$ such that ${V_{2}=\mathcal{W}_{Tail} \cup \mathcal{W}_{Head}}$ and
	$E_{2}= E_\text{\circled{$\dagger$}}$.
\end{myDef}

\begin{myDef}
	\label{def:time_homogeneous_t_link_graph}
	The \textbf{time-homogeneous $t$-link graph} ${\mathbb{G}_{2(t)}=(V_{2(t)},E_{2(t)})}$, ${t\!=\!N\!\!-\!\!1,N,\ldots}$, is defined by the vertex set ${V_{2(t)}=f_t(\mathcal{W}_{Tail}) \cup f_{t+1}(\mathcal{W}_{Head})}$ and the edge set
	$E_{2(t)}=\{e(w_i^t,w_j^{t+1}); e(w_i,w_j) \in E_\text{\circled{$\dagger$}}\}$.
\end{myDef}

\begin{myDef}
\label{def:graph_union_operation}
Take two graphs $\mathbb{G}_a=(V_a,E_a)$ and $\mathbb{G}_b=(V_b,E_b)$, where a vertex $v$ with label $l_v$ and an edge $e(v_1,v_2)$ with label $l_e$ are, respectively, defined by a pair $(v,l_v)$ and a triple $(v_1,v_2,l_e)$. A \textbf{union graph} of ${\mathbb{G}_a=(V_a,E_a)}$ and $\mathbb{G}_b=(V_b,E_b)$ is given by $\mathbb{G}=(V,E)=\mathbb{G}_a \oplus \mathbb{G}_b$, where $V=V_a \cup V_b$ and $E_a \cup E_b$.
\end{myDef}

\begin{myDef}
	\label{def:connected_repating_graph}
	The \textbf{connected repeating graph}~$\mathbb{D}_R^{t_a,t_b}$,  ${t_a,t_b\!=\!N\!\!-\!\!1,N,\ldots}$ and $t_a < t_b$, is obtained from the equation
	\vspace{-7pt}
	\begin{equation*}
	\label{eq:D_R}
	\mathbb{D}_R^{t_a,t_b}=	\mathbb{D}_{R(t_a)} \oplus \mathbb{G}_{2(t_a)} \oplus
	\mathbb{D}_{R(t_a+1)} \oplus \mathbb{G}_{2(t_a+1)}
	\oplus \ldots \oplus
	\mathbb{D}_{R(t_b-1)} \oplus \mathbb{G}_{2(t_b-1)}  \oplus \mathbb{D}_{R(t_b)}.
	\vspace{-7pt}
	\end{equation*}
	Henceforth fix $\mathbb{D}_R^{N\!-\!1,t}=\emptyset$, ${t=0,\ldots,N\!-\!2}$, and let $\mathbb{D}_R^{N\!-\!1,N\!-\!1}=\mathbb{D}_{R(N\!-\!1)}$.
\end{myDef}

Every time-slice $t$, ${t\!=\!N\!\!-\!\!1,N,\ldots}$, of a CEG $\mathbb{C}_{t}$ in~$\mathcal{F}(\mathbb{C})$, ${t\!=\!N\!\!-\!\!1,N,\ldots}$, has a similar stage structure to the one depicted in $\mathbb{D}_R$ because of the time-homogeneity of $\mathbb{C}$. However the number of positions associated with the last $\kappa$ time-slices can be smaller than the number of positions in $\mathbb{D}_R$, where: $\kappa=\min(t\!-\!N\!+\!\eta\!+\!1,N\!-\!\eta)$, ${t\!=\!N\!\!-\!\!1,N,\ldots}$; and $\eta(\mathbb{C})$, or simply $\eta$, is equal to~$1$, if $\mathcal{T}_{\!-1}=\emptyset$, and~$0$, otherwise. This happens because the probabilistic and graphical map identifying two situations by the same vertex in $\mathbb{C}_t$ only holds over a finite tree. To represent these last $\kappa$~time-slices, we next introduce two families of closure graphs and a vertex contraction operator~$\Phi$ for a coloured graph. This operator $\Phi$ enables us to introduce the topological simplifications in the set of vertices associated with the last $\kappa$~time-slices of a CEG $\mathbb{C}_t$ in $\mathcal{F}(\mathbb{C})$ that inherits the coloured graphical structure of subgraphs obtained from a DCEG~$\mathbb{C}$. By doing this, it obtains the position structure associated with the finite staged tree of~$\mathbb{C}_{t}$.  Let $\Lambda_u(v)$ be the set of direct paths that unfolds from a vertex $v$ in a directed acyclic graph.

\begin{myDef}
	\label{def:t_closure_graph}
	Take the following sets of graphs based on an $N\text{T-DCEG}$~$\mathbb{C}$:
	\begin{enumerate}[nosep]
		\item ${\mathbb{D}_{l(a)}^t=\mathbb{D}_R^{N\!-\!1,t} \oplus \mathbb{G}_{2(t)}}, {t=N\!-\!1,\ldots,2N\!-\!\eta\!-\!2}$; and
		\item ${\mathbb{D}_{l(b)}^t=\mathbb{D}_R^{t-N+1+\eta,t} \oplus \mathbb{G}_{2(t)}}$, ${t=2N-\!\eta\!-\!1,2N\!-\!\eta,\ldots}$.
	\end{enumerate} 
	For each ${\mathbb{D}_{l(i)}^t=(V_{l(i)}^t,E_{l(i)}^t)}$, ${i=a,b}$, the \textbf{$t$-closure graph}  $\mathbb{D}_{L(i)}^t$ is then constructed by merging the set of vertices
	$\mathcal{W}_{l(i)}^t=\{w \in V_{l(i)}^t; w \in f_{t+1}(\mathcal{W}_{Head})\}$
	into a single a node and relabelling it as $w_\infty$. If a terminating vertex $w_\infty$ already exists in $V_{l(i)}^t$, it is necessary only to merge the set $\mathcal{W}_{l(i)}^t$ into~$w_\infty$.  
\end{myDef}

\begin{myDef}
\label{def:vertex_contraction_operator}
Take a coloured directed acyclic graph $\mathbb{G}=(V,E)$, where each vertex has a numbered label and is associated with a time-slice~$t$, $t=-1,0,1,\ldots$. The \textbf{vertex contraction operator}~$\boldsymbol{\Phi}$ merges every two vertices~$v_a$ and~$v_b$ in~$V$ into a single vertex~$v_c$, $c=\min\{a,b\} $, if and only if they are coloured the same, they are associated with the same time-slice and there exists a bijection
\vspace{-5pt}
\begin{equation}
\label{eq:bijection_vertex}
\phi_v(v_a,v_b):\Lambda_u(v_a) \to \Lambda_u(v_b),
\vspace{-5pt}
\end{equation}
such that the ordered sequence of edge labels and edge colours in a path $\lambda \in \Lambda_u(v_a)$ equals the ordered sequence of edge labels and edge colours in the path $\lambda'=\phi_v(v_a,v_b)(\lambda),\lambda' \in \Lambda_u(v_b)$.
\end{myDef}

It is straightforward to see that 
$\mathbb{C}=
    \mathbb{D}_I \oplus \mathbb{G}_0 \oplus \mathbb{D}_H =
    \mathbb{D}_I \oplus \mathbb{G}_0 \oplus \mathbb{D}_{R}
        \oplus \mathbb{G}_2$.
Theorem \ref{theo:C_k} now asserts that every $\mathbb{C}_t$ in $\mathcal{F}(\mathbb{C})$, $t \!=\! 2N\!-\!2,2N\!-\!1,\ldots$, can be decomposed in four graphs $\mathbb{D}_I$, $\mathbb{G}_1$, $\mathbb{D}_R^{N\!-\!1,t\!-\!N\!+\!\eta}$ and $\mathbb{D}_L^{t}$; see also Figure~\ref{fig:NT-DCEG_as_CEG}. The graphs~$\mathbb{D}_I$ and~$\mathbb{G}_1$ correspond to the initialisation of our model. The graphs~$\mathbb{D}_R^{N\!-\!1,t-N+\eta}$ and~$\mathbb{D}_L^{t}$ are associated with the N-Markov time-homogeneity condition. These later graphs derived from~$\mathbb{D}_H =
\mathbb{D}_R \oplus \mathbb{G}_2$. In Figure~\ref{fig:NT-DCEG_as_CEG}, $\Lambda_L^t$ denotes the set of $f_{t\!-\!N\!+\!1}(\mathcal{W}_{Tail})$-to-$w_\infty$ paths in~$\mathbb{D}_{L}^t$. 

%\newpage
\begin{myTheorem}
\label{theo:C_k}
Take an $N\text{T-DCEG}$~$\mathbb{C}, N= 2,3,\ldots$. Then every CEG~$\mathbb{C}_t$ in~$\mathcal{F}(\mathbb{C})$ can be written as
\vspace{-5pt}
\begin{equation}
\label{eq:C_t}
\mathbb{C}_t=
\left\{
\begin{array}{ll}
	\Phi(\mathbb{D}_{I(\infty)}^t),
		& \text{if } t= 0,\ldots, N-2, \\
	\Phi(\mathbb{D}_I \oplus \mathbb{G}_1 \oplus \mathbb{D}_{L(a)}^{t})
 		& \text{if } t= N-1,\ldots, 2N-\!\eta\!-\!2,\\
	\mathbb{D}_I \oplus \mathbb{G}_1 \oplus \mathbb{D}_R^{N\!-\!1,t-N+\eta}
		\oplus \mathbb{D}_L^t
		 & \text{if } {t=2N-\!\eta\!-\!1,2N\!-\!\eta,\ldots}.
\end{array} \right.
\end{equation}
where $\mathbb{D}_L^t\!=\!\Phi(\mathbb{G}_{2(t-N+{\eta})} \oplus \mathbb{D}_{L(b)}^{t})$. 
%\vspace{-10pt}
\end{myTheorem}  
\begin{proof}
See \ref{app:C_k}.
\end{proof}

In contrast to a CEG, an $N\text{T-DCEG}$ provides us with a very expressive and summary framework for representing conditional statements in a dynamic environment. This is because all the  subgraphs $\mathbb{D}_R^{N\!-\!1,t\!-\!N\!+\!{\eta}}$ and $\mathbb{D}_L^{t}$ are summarised in the single subgraph~$\mathbb{D}_H$. Being based on an infinite tree, an $N\text{T-DCEG}$~$\mathbb{C}$ also avoids introducing unnecessary refinements of the position structure that its CEGs in $\mathcal{F}(\mathbb{C})$ might be forced to express. This happens because in a CEG~$\mathbb{C}_t$, $\mathbb{C}_t \in \mathcal{F}(\mathbb{C})$,  a position corresponds to a set of situations which always share a finite coloured subtree rather than an infinite one. The example below illustrates the concepts discussed in this section.

\begin{myExampleCont1}
\emph{
Recall the 2T-DCEG $\mathbb{C}$ depicted in Figure~\ref{fig:2T-DCEG_Radicalisation_3var}. Figure~\ref{fig:2T-DCEG_as_CEG} below shows how to construct the CEG~$\mathbb{C}_2$ using the different subgraphs derived from $\mathbb{C}$. As might be expected for a 2T-DCEG with no time-invariant information ($\mathcal{T}_{\!-1}=\emptyset$), the graph ${\mathbb{D}_I \oplus \mathbb{G}_1 \oplus \mathbb{D}_{R(1)}}$ is similar to $\mathbb{C}$ (Figure~\ref{fig:2T-DCEG_Radicalisation_3var}) except for the absence of cyclical temporal edges and the addition of superscripts $1$ to the vertices of~$\mathbb{D}_{R(1)}$.  The graphs $\mathbb{D}_I$ and $\mathbb{D}_{R(1)}$ are based, respectively, on the event tree $\mathcal{T}_0$ and the forest ${\mathcal{F}o_1=\{\mathcal{T}_1(s_{13+i}); i=0,\ldots,8\}}$ (Figure \ref{fig:ObjectOriented_InfiniteEventTree_Radicalisation_3var}). } 

\emph{The graph~$\mathbb{D}_R$ is topologically identical to the graph~$\mathbb{D}_{R(1)}$ but with the vertex superscript removed. The connective subgraph $\mathbb{G}_1$ is defined by the set of positions ${V_1\!=\!\{w_4,\ldots,w_9,w_{10}^1,\ldots,w_{15}^1\}}$ and the set of transition edges ${E_1\!=\!\{(w_4,w_{10}^1),\ldots,(w_9,w_{15}^1)\}}$. The subgraph~$\mathbb{D}_R^{1,1}$ is made up of only one repetition, $\mathbb{D}_{R(1)}$, of the subgraph~$\mathbb{D}_R$. So no bipartite graph~$\mathbb{G}_{2(t)}$ needs to be depicted in Figure~\ref{fig:2T-DCEG_as_CEG}. The subgraph~$\mathbb{D}_L^2$ can be directly obtained from $\mathbb{D}_R$ by relabelling its vertices and by merging the set of positions $\{w_{19},w_{20},w_{21}\}$, $\{w_{22},w_{23},w_{24}\}$ and $\{w_{25},w_{26},w_{27}\}$ of $\mathbb{D}_R$ into, respectively, the positions $w_{19}^2$, $w_{22}^2$ and $w_{25}^2$ of $\mathbb{D}_L^2$.  These positions are gathered into $\mathbb{D}_L^2$ because they have isomorphic unfolding developments over a single transition of a time-slice.}

%\vspace{23pt}
\begin{figure}[t] 
	\begin{center}
		\includegraphics[scale=0.69,angle=-90,origin=c,trim=50 13 30 -90]{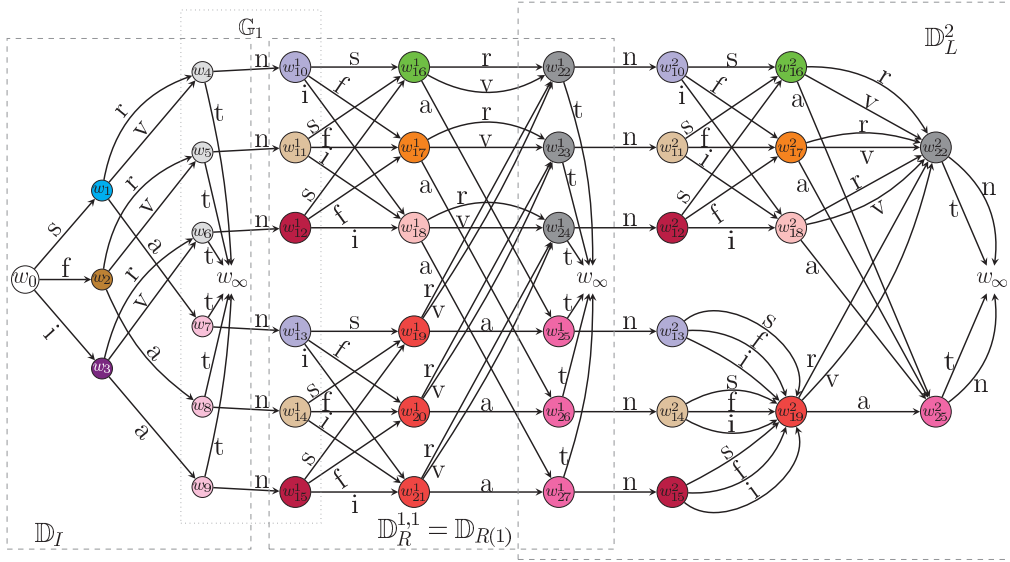}
	\end{center}
	\vspace{-75pt}
	\caption{The CEG $\mathbb{C}_2$ associated with the 2T-DCEG depicted in Figure \ref{fig:2T-DCEG_Radicalisation_3var} \label{fig:2T-DCEG_as_CEG}}
	\vspace{-5pt}
\end{figure}

\end{myExampleCont1} 

Theorem~\ref{theo:C_k} tell us that in order to build an $N\text{T-DCEG}$~$\mathbb{C}=(V,E)$ supported by an infinite staged tree~$\mathcal{ST}_{\!\!\!\infty}$ it is necessary only to construct its corresponding finite staged tree~$\mathcal{ST}_{\!\!\!2N-\eta-1}$. This happens because the subgraph~${\mathbb{D}_I \oplus \mathbb{G}_1 \oplus \mathbb{D}_R^{N\!-\!1,N\!-\!1}}$ of~$\mathbb{C}_{2N-\eta-1}$ supported by~$\mathcal{ST}_{\!\!\!2N{-\eta-1}}$ is isomorphic to the subgraph~$\mathbb{D}_I \oplus \mathbb{G}_0 \oplus \mathbb{D}_R$ of~$\mathbb{C}$. Remember that ${\mathbb{C}= \mathbb{D}_I \oplus \mathbb{G}_0 \oplus \mathbb{D}_R \oplus \mathbb{G}_2}$. The subgraph~$\mathbb{D}_I \oplus \mathbb{G}_0 \oplus \mathbb{D}_R$ completely defines the vertex set~$V$ and the edge set~${E^-}$ of~$\mathbb{C}$. The graph~$\mathbb{G}_2$ just adds the set~$E_\text{\circled{$\dagger$}}$ of cyclical temporal edges to~$\mathbb{C}$.

This result enables us to propose a method to build an $N$T-DCEG model. Algorithm~\ref{alg:NTDCEG_algorithm} describes how to implement it computationally. Initially, domain experts need to provide the event tree~$\mathcal{T}_{\!-1}$ corresponding to the time-invariant information and the event tree~$\mathcal{T}$ associated with each time-slice of the observed process. From this elicitation it is possible to construct a TOG$(\mathcal{T}_{\!-1},\mathcal{T},N\!-\!1)$. Domain experts can then describe the stage structure of the first $N\!-\!1$ time-slices using this as a framework to obtain~$\Delta(\mathcal{ST}_{\!\!\!N\!-\!1})$.

\input{NTDCEG_Algorithm}

Let $\Gamma_*$ be the set of objects~$\Delta(\mathcal{ST}_{\!\!\!N\!-\!1}(s_i))$, where $s_i$ is a situation corresponding to a leaf~$l_k$ of~$\mathcal{ST}_{\!\!\!N\!-\!2}$.
Note that the time-homogeneity after time~$N\!-\!1$ and the periodicity spanned by a TOG$(\mathcal{T}_{\!-1},\mathcal{T})$ makes one of this object unfold from each leaf node of~$\Delta(\mathcal{ST}_{t})$, $t=N\!-\!1,\ldots,2N{\!-\eta\!-\!2}$, to obtain recursively $\Delta(ST_{t+1})$. It then follows that the object~$\Delta(\mathcal{ST}_{\!\!\!N\!-\!1})$ provides us with all elements to construct~$\Delta{(\mathcal{ST}_{\!\!\!2N{-\eta-1}})}$ using equation~\ref{eq:recursive_structure_tree}. For this purpose, for all~$t$, $t=N\!-\!1,\ldots,2N{\!-\eta\!-\!2}$, take $\Upsilon(\mathcal{ST}_t)=\{\Upsilon_k\}$, where $\Upsilon_k$ is the set of leaves~$l_j$ of~$\Delta(\mathcal{ST}_{t})$ such that the concatenation of events~$\xi(l_j,N-1)$ is equal to~$\xi(s_i,N-1)$ and $s_i$~corresponds to the leaf node~$l_k$ of~$\mathcal{ST}_{\!\!\!N\!-\!2}$. Also define $h_t: \Upsilon(\mathcal{ST}_t) \to \Gamma_*$, where
$h_t(\Upsilon_k)=\Delta(\mathcal{ST}_{\!\!\!N\!-\!1}(s_i))$
such that $s_i$~corresponds to the leaf node~$l_k$ of~$\mathcal{ST}_{\!\!\!N\!-\!2}$. We can then use the staged tree~$\mathcal{ST}_{\!\!\!2N\!{-\eta-1}}$ to obtain the position structure~$W$ associated with the first $N\!-\!1$ time-slices.

Now return to the staged tree depicted by~$\Delta(\mathcal{ST}_{\!\!\!N\!-\!1})$ and obtain the graph~$\mathbb{C}_*$ (Definition~\ref{def:star_graph}). This graphical transformation implicitly defines the cyclic temporal edges of the $N$T-DCEG~$\mathbb{C}$. In fact, it is straightforward to verify that the situations~$s_a$ and $s_b$ of~$\mathcal{ST}_{\!\!\!\infty}$ corresponding, respectively, to~$s_i$ and~$l_j$ of~$\mathcal{ST}_{\!\!\!N\!-\!1}$ in Definition~\ref{def:star_graph} are in the same position. This is because the periodicity and time-homogeneity is assumed.  

\begin{myDef} 
	\label{def:star_graph}
	A \textbf{graph}~$\mathbb{C}_*$ is obtained from a staged tree~$\mathcal{ST}_{\!\!\!N\!-\!1}$ when each of its leaf~$l_j$ is merged to the situation~$s_i$ such that \vspace{-7pt}$$\xi(l_j,N\!-\!1)=\xi(s_i,N\!-\!1),\vspace{-7pt}$$
	 where $s_i$ corresponds to a leaf~$l_k$ of~$\mathcal{ST}_{\!\!\!N\!-\!2}$.  
\end{myDef}

To obtain the $N$T-DCEG~$\mathbb{C}$ it is then necessary only to combine together the vertices of~$\mathbb{C}_*$ according to the position structure~$W$.
Finally, domain experts needs to elicit the conditional probability distributions based on~$\mathbb{C}$. Note that as many context-specific conditional independence are depicted by~$\mathbb{C}$ less effort will be required from domain experts in comparison to this quantitative elicitation using a DBN or the staged tree itself. This justifies to construct first the stage structure not only in terms of time but also in terms of reliability since domain knowledge is more robust as regard the qualitative structure than its quantification.

%%%%%%%%%%%%%%%%%%%%%%%%%%%%%%%%%%%%%%%%%%%%%%%%%%%%%%%%%%%%%%
\section{Reading Conditional Independence in an $N$T-DCEG}
\label{sec:conditional_independence}

Domain experts often describe a process based on sequences of events that characterise situations where a unit can be at a particular time.  As illustrated in~Examples~\ref{ex:Radicalisation_Static_3variables} and~\ref{ex:Radicalisation_Dynamic_3variables}, in these cases they do not immediately reason using random variables. Therefore, it is important to develop a sound methodology of reading conditional independences from the topology of an $N$T-DCEG graph using a set of random variables identified only after the model elicitation has taken place. In the future, this will be even more important when automatic model selection algorithms based on data become available and provide analysts with an $N$T-DCEG model for exploration and interpretation.
In~\cite{Collazo.2017,Collazo.Smith.2018}
	 we develop some techniques to formally construct random variables based on the topology of an $N$T-DCEG graph. Here our focus is on reading the existing conditional independences at time~$t$ given that some past events are known.

It is obvious that the conditional independences embedded into~$\mathbb{C}$ at time-slices~$t,t=0,\ldots, N\!-\!1$, can be directly read from its subgraph~$\mathbb{D}_I$. Reading the conditional independence at the subsequent time-slices requires us to develop a specific tool. For this purpose, we will first obtain a result associated with the family of CEGs~$\mathcal{F}(\mathbb{C})$ connecting a time-slice~$t$, ${t=N,N\!+\!1,\ldots}$, with time-slice~$N\!-\!1$. We will subsequently translate this result to the ${N\text{T-DCEG}}$ topology.

Let $w_i(t)$ be the set of situations that happen at time~$t$ and are in the position~$w_i$ of a DCEG/CEG . Also let $\Xi_c(w(t),N)=\{\xi(s,N);s \in w(t)\}$ denote a set of all sequences of events $\xi(s,N),s \in w(t)$, that happen along each walk from the root position $w_0$ to $w(t)$ whose events from time $0$ to $t-N$ are excluded. Theorem \ref{the:legend} guarantees that for every time-slice~$t$, $t=N, N\!+\!1,\ldots$,  there exists a sufficient large time-slice~$T$ such that in~$\mathbb{C}_T$ every position~$w_a^*$ at time-slice~$t$ that has a parent at previous time-slice~$t\!-\!1$ corresponds to a position~$w_b^*$ at time-slice~$N\!-\!1$. Moreover, the set of sequences of events that happened in the last~${N\!-\!1}$ time-slices preceding~$w_a^*$ and~$w_b^*$ are the same.

\begin{myTheorem}
Take an NT-DCEG $\mathbb{C}=(V,E)$ and define the set of positions $\mathcal{W}_{Head}$ according to Definition~\ref{def:temporal_edges}. In a CEG ${\mathbb{C}_T=(V_T,E_T)}$, ${\mathbb{C}_T \in \mathcal{F}(\mathbb{C})}$ and~${T = 2N\!-\!{\eta}, 2N\!-\!{\eta\!+\!1},\ldots}$, for every position~$w_a^*(t)$ in the set~$f_t(\mathcal{W}_{Head})$, ${f_t(\mathcal{W}_{Head}) \subseteq V_T}$ and~${t=N,\ldots,T\!-\!N\!+\!{\eta}}$, we have that
\vspace{-7pt}
\begin{equation}
\label{eq:legend}
\Xi_c(w_a^*(t),N-1)=\Xi_c(w_b^*(N-1),N-1),
\vspace{-7pt}
\end{equation}
where $w_b^*(N\!-\!1)= f_{N-1}(f_{t}^{-1}(w_a^*(t))) \in V_T$.
\label{the:legend}
\vspace{0pt}
\end{myTheorem}
\begin{proof}
See \ref{app:legend}.
\end{proof}

It then follows that we can directly interpret the conditional independences at time~$t$, ${t= N, N\!+\!1,\ldots}$, using~$\mathbb{C}$ as we do at time-slice~$N\!-\!1$.} This is possible for three reasons. First, since every process represented by an $N\text{T-DCEG}$ is $N\text{-Markov}$, we only need information over the last $N-1$ time-slices. Second, Theorem \ref{theo:C_k} assures us that the conditional independences at a time~$t,t= N\!-\!1,N,\ldots$, are depicted by a subgraph $\mathbb{D}_{R(t)}$ that is isomorphic to a subgraph $\mathbb{D}_R$. Third, Theorem~\ref{the:legend} asserts that every position at the beginning of time-slices~${N\!-\!1}$ and~$t$, ${t= N, N\!+\!1,\ldots}$, have the same information about the past events in the last $N\!-\!1$ time-slices. Note that this information is completely represented in~$\mathbb{D}_I$. The example below illustrates how to use this methodology for reading conditional independence statements from the topology of an $N$T-DCEG.

\begin{myExampleCont1} 
	\emph{
		Return to Example~\ref{ex:Radicalisation_Dynamic_3variables}. As for CEG graph in Example~\ref{ex:Radicalisation_Static_3variables}, the construction of the $N$T-DCEG graph (Figure~\ref{fig:2T-DCEG_Radicalisation_3var}) now helps us to identify the variables~$N(t)$, $R(t)$ and~$T(t)$ that takes values over each time~$t$, $t=0,1,\ldots$. Using the subgraph~$\mathbb{D}_I$ associated with the initial time-slice~of the $N$T-DCEG graph (Figure~\ref{fig:2T-DCEG_Radicalisation_3var}) as a legend to analyse the subsequent time-slices and exploring the fact that the positions~$w_{19},w_{20}$ and~$w_{21}$ are coloured the same, the following context-specific conditional independence stands out: variable~$R(t\!+\!1)$ is conditionally independent of variable~$N(t\!+\!1)$, $t=0,1,\ldots$, given that the variables~$R(t)$ and~$T(t)$ assume, respectively, values equal to~$a$ and~$n$. So to summarise,\vspace{-9pt}
		$$R(t\!+\!1) \independent N(t\!+\!1) | (R(t)=a,T(t)=n), \text{ for } t=0,1,\ldots. \vspace{-9pt}$$
		Analogous reasoning leads us to read the following conditional independences:
		\begin{itemize}[nosep]
			\item $T(0) \independent N(0) | R(0)$ and $T(0) \independent (N(0),R(0)) | R(0) \neq a$.
			\item $N(t\!+\!1) \independent R(t) | (N(t), T(t)=n)$, for $t=0,1,\ldots$.
			\item $R(t\!+\!1) \independent N(t) | (R(t), T(t)=n, N(t\!+\!1))$, for $t=0,1,\ldots$.
			\item $R(t\!+\!1) \independent N(t) | (R(t)=a, T(t)=n)$, for $t=0,1,\ldots$.
			\item $R(t\!+\!1) \!\independent\! (N(t),R(t)) | (R(t) \!\neq\! a, T(t)\!=\!n,N(t\!+\!1))$, for~${t=0,1,\ldots}$.	
			\item $T(t\!+\!1) \!\independent\! (N(t),R(t),N(t\!+\!1)) | (T(t)\!=\!n,R(t\!+\!1))$, for ${t=0,1,\ldots}$.
			\item $T(t\!+\!1) \!\independent\! (N(t),R(t),N(t\!+\!1),R(t\!+\!1) | (R(t+1) \neq a,T(t)=n)$, for~${t=0,1,\ldots}$.
		\end{itemize}
	}
	
	\emph{Based on the discussion presented in the example of Section~\ref{subsec:NT-DCEG}, we can also construct a different set of variables~$N$, $R^*$ and~$T$, where $R^*$ is a categorical variable indicating if an inmates is adopting radicalisation. Replacing~$R$ by~$R^*$, we can obtain the same conditional statements outlined above.}
\end{myExampleCont1}

The analysis of how a process can unfold $s$ time steps ahead from time~$t$ given a particular set of past events~$\mathcal{E}$ needs a little more care because of the cyclical temporal edges. Observe that the set~$\mathcal{E}$ corresponds to a set of positions~${\mathcal{W}_\mathcal{E}}$ at the beginning of time~$t$, i.e. every position in~${\mathcal{W}_\mathcal{E}}$ has at least one parent in time~$t-1$,  $t=0,1,\ldots$, or ${\mathcal{W}_\mathcal{E}}=\{w_0\}$. If interest lies in a time-slice~${t\!+\!s}$ that happens within the first~$N\!-\!1$ time-slices (${t\!+\!s \! = \{-1,0,\ldots, \! N\!-\!1\}}$), the analysis using an ${N\text{T-DCEG}}$ is simplified by discarding the walks that do not unfold from~$\mathcal{W}_\mathcal{E}$.  The same procedure also applies if the focus is on the present time~$t$, ${t=N\!-\!1,N,\ldots}$. In this case, we have that $s=0$ and ${\mathcal{W}_\mathcal{E} \! \subseteq \! \mathcal{W}_{Head}}$. 

When $s$ and $t+s$ are, respectively, greater than $0$ and $N$, to explore how a process might unfold over $s$ time steps after the actual time $t$ given a particular set of past events $\mathcal{E}$ needs more attention.  However the task can be easily simplified if the transition matrix $\boldsymbol{M}$ associated with the Markov Chain projection of the elicited $N$T-DCEG (Section \ref{subsec:NT-DCEG}) is used.

This assumes that $t$ is greater than $N-2$. We are then able to identify from $\mathcal{E}$ the set~$\mathcal{W}_{\mathcal{E}(s)}$ of positions~$w$, $w \in \mathcal{W}_{Head}$, a unit may be in at the beginning of time $t+s$. Let $\boldsymbol{w}(t)$ and $p(\boldsymbol{w}(t))$  be, respectively, a binary vector that represents this location information and its corresponding probability vector at the beginning of time~$t$. Then
\begin{equation}
\label{eq:step_ahead_repeating}
p(\boldsymbol{w}(t+s))=p(\boldsymbol{w}(t)) \times \boldsymbol{M}^{{s}}.
\end{equation}      

Now based on the vector $p(\boldsymbol{w}(t+s))$ we can define~$\mathcal{W}_{\mathcal{E}(s)}$, $\mathcal{W}_{\mathcal{E}(s)} \! \subseteq \! \mathcal{W}_{Head}$, and then use the same framework described for $s=0$ to analyse what might happen at time-slice $t+s$. Note that when $t \in \{-1,0,\ldots,N\!-\!2\}$ and $t+s \in \{N,N\!+\!1,\ldots\}$, before applying Equation \ref{eq:step_ahead_repeating} we first need to project our current information at time $t$ into the future time-slice $N\!-\!1$. We therefore need to use the transitions depicted in the initial subgraph $\mathbb{D}_I$ in order to find the set of positions that a unit can be at time-slice $N\!-\!1$ based on its possible positions at time $t$. In this case, it follows that  
\begin{equation}
\label{eq:step_ahead_initial}
p(\boldsymbol{w}(t+s))=
	p(\boldsymbol{w}(t)) \times
		\boldsymbol{M}_t \times
		\boldsymbol{M}^{t+s-N+{1}},
\end{equation}
where $\boldsymbol{M}_t, t = -1,0,\ldots,N\!-\!2$, is a transition matrix associated with the positions in $\mathbb{D}_I$ from the beginning of time-slice $t$ to the beginning of time-slice $N\!-\!1$.

\begin{myExampleCont1}
	\emph{
		Assume that in his most recent period $t$ in prison an inmate adopting radicalisation kept intense social contacts with extremist recruiters. We are concerned about what might happen to him at the next time step were he to stay in the same prison. In this case, we have that the set of past events $\mathcal{E}$ corresponds to the set of events~${E(t)=\{i,a,n\}}$
		at time-slice $t$, and our focus is on the developments that might happen at time-slice $t+1$. Using the 2T-DCEG elicited in Figure~\ref{fig:2T-DCEG_Radicalisation_3var} as representative of this process, the possible future developments associated with the event set~$\mathcal{E}$, where $\mathcal{W}_\mathcal{E}=\{w_{15}\}$, is highlighted in Figure~\ref{fig:2T_intense_ready_no} below. Two points stand out. First, our model implies that observing the social contacts of the target inmate at time ${t\!+\!1}$ will not provide any additional information about his degree of radicalisation or his probability to remain in prison within this time interval.}
	
	\begin{figure}[t]
		\begin{center}
			\includegraphics[scale=0.9,angle=-90,origin=c,trim=25 19 0 -135]
			{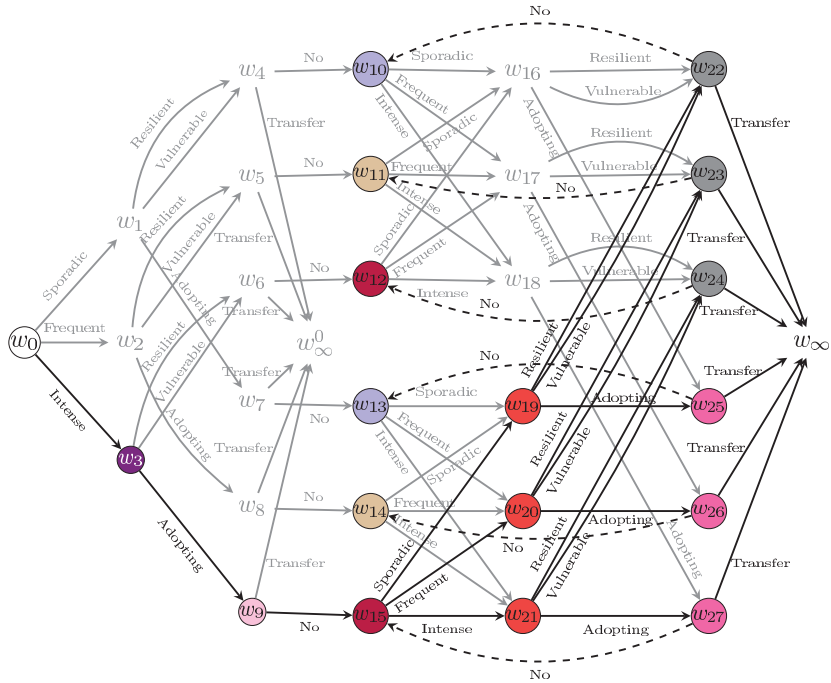}
		\end{center}
		\vspace{-30pt}
		\caption{The 2T-DCEG associated with Example \ref{ex:Radicalisation_Dynamic_3variables} when a radical prisoner maintains intense social contacts with other radical inmates and is not transferred.  \label{fig:2T_intense_ready_no}}
		\vspace{-10pt}
	\end{figure}
	
	\emph{
		Second, assuming the transition matrix $\boldsymbol{M}$ given in equation~\ref{eq:transition_matrix_example}, we then have
		$p(\boldsymbol{w}(t+2))
			=
				p(\boldsymbol{w}(t+1)) \times	\boldsymbol{M}
			=
				(0.15,0.06,0.14,0.09,0.11,0.28,0.17)$,
		where $p(\boldsymbol{w}(t+1))=(0,0,0,0,0,0,1)$.
		Therefore,  at the end of the time-slice $t\!+\!1$ the prisoner adopting radicalisation could arrive in any possible position in the set ${\mathcal{W}_{\mathcal{E}(1)}=\{w_{\infty},w_i;i\!\!=\!\!10,\ldots,15\}}$.
		}
	
	\emph{		
		This indicates that the prison managers might lose track of him.  For example, this prisoner could intentionally reshape his social networks at time-slice $t+1$ to disguise his extremist ideology. The consequence would be that at the beginning of time-slice $t+2$ he could be in position~$w_{10}$ or~$w_{13}$ with probability~$6\%$ and $11\%$, respectively. In this case, he could also be in position~$w_{11}$ or~$w_{14}$ with even higher probabilities ($14\%$ and $28\%$, respectively). Note that inmates at positions $w_{10}$ and $w_{13}$ have the same social pattern as well as at positions $w_{11}$ and $w_{14}$. On the other hand the inmate only has a $17\%$ chance of being identified in the same position~$w_{15}$ at the begining of time~$t\!+\!2$. Moreover, classifying an inmate as adopting radicalisation is a challenging task and error prone. At time $t+2$ these facts might misguidedly prompt prison managers to downgrade the inmate's classification as adopting radicalisation given at time $t$ and so reduce the monitoring mechanisms over him. So the reasoning the 2T-DCEG provokes is useful: it might stimulate some pro-active response immediately to deradicalise the inmate at time $t+1$ or at least to monitor him closely for a long time.} 
	
\end{myExampleCont1}  

%%%%%%%%%%%%%%%%%%%%%%%%%%%%%%%%%%%%%%%%%%%%%%%%%%%%%%%%%%%%%%%
\section{Conclusions}
\label{sec:Conclusion}

An ${N\text{T-DCEG}}$ is able to encode many asymmetric and context-specific independence structures. We have shown that these can be read directly using the algorithmic tools we developed in Section \ref{sec:conditional_independence}. In analogy with the interrogation methods used for a BN, the deductions from a DCEG model can be fed back to domain experts for verification or criticism. This process will continue until the hypothesised model is requisite \cite{Smith.2010,Phillips.1984,Phillips.2010}; i.e until no obvious inadequacies in the implications of the model could be found. In this way the plausibility of the qualitative implications of a hypothesised model can be examined \emph{before} any costly quantitative population of the graphical probability model takes place.

We have shown that this new dynamic class of models is compact. It is therefore sufficient to provide us with a framework for fast propagation of evidence and for model selection. Adopting a Bayesian approach \cite{Barclay.etal.2015}, for example, we can directly extend the propagation algorithm \cite{Thwaites.etal.2008} and model selection methods \cite{Freeman.Smith.2011a,Barclay.etal.2013a,Cowell.Smith.2014} that exists for CEGs. In particular a heuristic strategy for model search using non-local priors looks promising for model spaces based on event trees \cite{Collazo.Smith.2016}. Such algorithms will be reported in a forthcoming work.

In the future we plan to extend our object-recursive approach and the \emph{tree} objects developed in Section \ref{subsec:Event_Tree} to define a continuous time DCEG ($\text{CT-DCEG}$). This new family of DCEGs will enable us to systematically construct models primarily designed to describe how and when events might happen during irregular time-step transitions. Such CT-DCEG models should then extend the continuous time BN~\cite{Nodelman.etal.2002, Nodelman.etal.2003} and further explore the link between a general DCEG with holding times and semi-Markov processes~\cite{Barclay.etal.2015}. 

Finally, we will demonstrate in a later paper that the process-driven objects defined here can also be used to define an Object-Oriented CEG/DCEG and so provide a generalisation of Object-Oriented BNs \cite{Koller.Pfeffer.1997,Bangso.Wullemin.2000}. We believe that this development will facilitate the knowledge engineering process using event trees and the reusability of computational codes, particularly in large and complex real-world applications.

\newpage
%%%%%%%%%%%%%%%%%%%%%%%%%%%%%%%%%%%%%%%%%%%%%%%%%%%%%%%%%%%%%%%
\appendix
%%%%%%%%%%%%%%%%%%%%%%%% 
\section{Proof of Theorem \ref{the:periodicity_finite_DCEG}}
\label{app:theorem_periodicity_finite_DCEG}

\noindent
\emph{
	If a DCEG is finite, then its supporting event tree is periodic after some time~T.
	\vspace{5pt}} 

Assume that a finite DCEG is supported by an event tree that it is non-periodic for every $t, t=0,1,\ldots$. By Lemma~\ref{lem:PET} below, it will then follow that for an infinite subsequence of time-slices~$T_i, i=1,2,\ldots$, there must be at least one situation~$s_{a_i}(T_i)$, such that no situation~$s_b(t)$, $t=0,\ldots,{T_i\!-\!1}$, satisfies simultaneously the local and global conditions required by the bijection~$\phi(s_a,s_b)$ in Equation~\ref{eq:bijection_periodicity_tree}. 

So, by definition~\ref{def:position}, every situation $s_{a_i}(T_i), i=1,\ldots$, must be in a different position~$w_{c_i}$. Therefore the number of positions in any staged tree supported by this event tree is infinite. This contradicts the hypothesis. It follows that the supporting event tree of a finite DCEG must therefore be PET-$T$ for some~$T$.

%\newpage
\begin{myLemma}
	\label{lem:PET}
	Take an infinite event tree~$\mathcal{T}_{\!\infty}$. If the global and local conditions required by the bijection~$\phi(s_a,s_b)$ in Definition~\ref{def:periodic_event_tree} is satisfied for some time~$T'$, then there must exist a time~$T$, $T \!\in\! \{T',T'\!+\!1,\ldots\}$, such that $\mathcal{T}_{\!\infty}$ is PET-$T$. 
\end{myLemma}

\begin{proof}
	If $\mathcal{T}_{\!-1}=\emptyset$, then $\mathcal{T}_{\!\infty}$ is PET-$T'$. Otherwise, hypothesise that $\mathcal{T}_{\!\infty}$ is not a PET-$T$ for any~$T$, $T \!\in\! \{T',T'\!+\!1,\ldots\}$. Recall that each situation $s_{l_i}(-1)$ corresponding to a leaf of~$\mathcal{T}_{\!-1}$ defines a type of units observed in the process. Hence, for some situation $s_{l_i}$, there must exit an infinite sequence of situations~$s_{a_j}(T_j)$, $j=1,2,\ldots$, such that:  the sequence of $T_j$, $j=1,2,\ldots$, is monotonic increasing; $s_{a_j} \in \mathcal{T}_{\infty}(s_{l_i})$, $j=1,2,\ldots$; and every $s_{a_j}$ does not satisfy the time-invariant condition required by the bijection in Equation~\ref{eq:bijection_periodicity_tree}. By assumption, for every~$s_{a_j}(T_j)$ there must be a situation~$s_{b_j}(t)$, $t=0,\ldots,T'\!-\!1$, such that the bijection~$\phi(s_{a_j},s_{b_j})$ in Definition~\ref{def:periodic_event_tree} satisfies the local and global conditions and $s_{b_j} \in \mathcal{T}_{T'\!-\!1}(s_{l_m})$, $m \neq i$. By our hypothesise, there must exist an infinite subsequence~$(s_{c_k})_k$ of~$(s_{a_j})_j$ such that $s_{b_k} \in V_{T'\!-\!1}(s_{l_m}) \backslash V_{T'\!-\!1}(s_{l_i}) \cup \{s_{b_K}; K=1,\ldots,k\!-\!1\}$, where: $m \neq i$; $V_{T'\!-\!1}(s_{l_m})$ is the vertex set of~$\mathcal{T}_{T'\!-\!1}(s_{l_m})$; and $(s_{b_k})_k$ is an infinite subsequence of $(s_{b_j})_j$. If this was not true, then the set of situations $s_{b_j}$ would be fix after some time~$T$, $T \!\in\! \{T',T'\!+\!1,\ldots\}$. Hence $\mathcal{T}_{\!\infty}$ would be PET-$T$. However, since $V_{T'\!-\!1}(s_{l_m})$ is finite there is not such subsequence~$(s_{c_k})_k$. The result then follows by contraction.
\end{proof}

%%%%%%%%%%%%%%%%%%%%%%%% 
\section{Proof of Theorem \ref{the:finite_DCEG}}
\label{app:theorem_finite_DCEG}

\noindent
\emph{
	A DCEG supported by a staged tree $\mathcal{ST}_{\!\infty}$ is finite if and only if, for some time-slice~$t_b$, every situation~$s_b$ in~$l(\mathcal{ST}_{\!t_b})$ is in the same position as a situation~$s_a$ in~$\mathcal{ST}_{\!t_a}$, $t_a=0,\ldots,{t_b}$.
	\vspace{3pt}}

Assume that a finite DCEG is supported by a staged tree~$\mathcal{ST}_{\!\infty}$ such that no time-slice~$t_b$ satisfies the condition described in Theorem~\ref{the:finite_DCEG}. It will then follow that for every time-slice~$t_b, t_b=0,1,\ldots$, there must be at least one situation~$s_b(t_b\!+\!1)$ in~$l(\mathcal{ST}_{\!t_b})$ that is in position $w_{t_b}$, such that $w_{t_b}$ does not merge any situation~$s_a(t_a)$, $t_a=0,\ldots,{t_b}$. Therefore the number of positions in this staged tree is infinite. This contradicts the hypothesis. It follows that the supporting staged tree of a finite DCEG must satisfy the condition described in Theorem~\ref{the:finite_DCEG}.

Conversely, assume that for some time-slice~$t_b$, every situation~$s_b$ in~$l(\mathcal{ST}_{\!t_b})$ is in the same position as a situation~$s_a$ in~$\mathcal{ST}_{\!t_a}$, $t_a=0,\ldots,{t_b}$. Now suppose by absurd that the corresponding DCEG supported by $\mathcal{ST}_{\!\infty}$ is infinite; i.e., $\mathcal{ST}_{\!\infty}$ has an infinite number of positions. Let $t_c$, be the first time-slice after~$t_b$ such that a situation $s_c(t_c)$ is in a position~$w$ that does not merge any situation in the previous time-slices. Since we hypothesised that $\mathcal{ST}_{\!\infty}$ has an infinite number of positions, there must exist such position~$w$. Also let $a(s_i,t)$ be the antecedent situation of a situation $s_i$ such that $a(s_i,t)$~in~$l(\mathcal{ST}_t)$. Now take $\mathcal{ST}_{\!\infty}(a(s_c,t_b))$. It then follows from the assumed condition that $\mathcal{ST}_{\!\infty}(a(s_c,t_b))$ is graphically isomorphic with $\mathcal{ST}_{\!\infty}(s_a)$, for some situation~$s_a$ in~$\mathcal{ST}_{\!t_a}$. So, there is a situation $s_d(t_d)$, such that $t_d \in \{0,\ldots,t_c\!-\!1\}$, and $s_c$ and $s_d$ are in the same position~$w$. But this is a contraction since $s_c$ is the early situation in position~$w$. Therefore the DCEG must be finite.

%%%%%%%%%%%%%%%%%%%%%%%% 
\section{Proof of Theorem \ref{the:THST_finite_DCEG}}
\label{app:THST_finite_DCEG}

\noindent
\emph{
	Every time-homogeneous staged tree after time $T$ has an associated DCEG with a finite graph.
	\vspace{3pt}}

In any event tree, let $\xi(s_a,s_b)$ be the sequence of events that happen along the finite path between the situations $s_a$ and $s_b$, where $s_b$ descends from~$s_a$. Also let $a(s_i,t)$ be the antecedent situation of a situation $s_i$ such that ${a(s_i,t)} \in {l(\mathcal{T}_t)}$. By assumption, the event tree is a PET-$T$. So, for every situation $s_i$, ${s_i \in l(\mathcal{T}_{2T+1}})$, there is a situation~$s(t)$, ${t =0,1,\ldots,T}$, such that there exists a bijection between the infinite event trees $\mathcal{T}_\infty(a(s_i,T))$ and $\mathcal{T}_\infty(s(t))$ satisfying the conditions given in Definition~\ref{def:periodic_event_tree}. It then follows that there is a situation ${s_*(t^*)}$, ${t^*=t+T+1}$, such that: ${s_*(t^*) \in \mathcal{T}_\infty(s(t))}$;
${s_*(t^*)\in l(\mathcal{T}_{t^*-1})}$; ${\xi(s(t),s_*(t^*))\!=\!{\xi(a(s_i,T),s_i)}}$;
the infinite event trees $\mathcal{T}_\infty(s_i)$ and $\mathcal{T}_\infty(s_*(t^*))$ are graphical isomorphic in the sense of Equation~\ref{eq:bijection_periodicity_tree}. Time-homogeneity after time $T$ then tell us that there is also a probabilistic isomorphism between the primitive probabilities associated with $\mathcal{T}_\infty(s_i)$ and $\mathcal{T}_\infty(s_*(t^*))$. So, $s_i$ and $s_*(t^*)$, ${t^* \in \{1,\ldots,2T+1\}}$, are in the same position. DCEG must therefore be finite by Theorem~\ref{the:finite_DCEG}. 

%%%%%%%%%%%%%%%%%%%%%%%% 
\section{Proof of Theorem~\ref{theo:C_k}}
\label{app:C_k}

\noindent
\emph{
	Take an NT-DCEG $\mathbb{C}, N= 2,3,\ldots$. Then every CEG~$\mathbb{C}_t$ in~$\mathcal{F}(\mathbb{C})$ can be written as
	\vspace{-5pt}
	\begin{equation*}
	\label{eq:C_t}
	\mathbb{C}_t=
	\left\{
	\begin{array}{ll}
	\Phi(\mathbb{D}_{I(\infty)}^t),
	& \text{if } t= 0,\ldots, N-2, \\
	\Phi(\mathbb{D}_I \oplus \mathbb{G}_1 \oplus \mathbb{D}_{L(a)}^{t})
	& \text{if } t= N-1,\ldots, 2N{-\!\eta\!-\!2},\\
	\mathbb{D}_I \oplus \mathbb{G}_1 \oplus \mathbb{D}_R^{N\!-\!1,t-N+{\eta}}
	\oplus \mathbb{D}_L^t
	& \text{if } {t=2N{-\!\eta\!-\!1},2N\!-{\!\eta},\ldots}.
	\end{array} \right.
	\end{equation*}
	where $\mathbb{D}_L^t\!=\!\Phi(\mathbb{G}_{2(t-N+{\eta})} \oplus \mathbb{D}_{L(b)}^{t})$. 
	\vspace{3pt}}

Denote by $U_t$ the stage structure of each time-slice $t$ of an $N\text{T-DCEG}$~$\mathbb{C}$. By definition, a CEG~$\mathbb{C}_T$ in~$\mathcal{F}(\mathbb{C})$ has the same stage structure $U_t$ for every time-slice~$t$, ${t=-1,0,\ldots,T}$. Let $W_t^C$, and~$W_t^{C_T}$, ${t=-1,0,\ldots,T}$, be, respectively, the position structure of $\mathbb{C}$ and $\mathbb{C}_T$ associate with time-slice $t$. We therefore need to prove that $W_t^{C_T}=W_t^C$, when ${t\in \{-1,0,\ldots,T\!-\!N\!+\!{\eta}\}}$ and
${T\in\{2N\!-\!{\eta}\!-\!1,2N\!-\!{\eta},\ldots\}}$. If this is true, the result then follows by the definition of the subgraphs introduced in Section~\ref{sec:CEG_NTDCEG}.

To do this, take two situations $s_a$ and $s_b$ at time-slice $t$,  ${t\!=\!-\!1,\ldots,T\!\!-\!\!N\!\!+\!\!\eta}$, such that $s_a$ and $s_b$ are, respectively, in two different positions~$w_{a}$ and~$w_{b}$ in~$W_t^C$ but at the same stage $u_j \in U_t$. Note that we only need to consider situation at the same time-slice because $\mathcal{F}(\mathbb{C})$ is a family of CEGs that are constructed using the concept of $\infty$-position. Hypothesise that  $s_a$ and $s_b$ are in the same position~$w_c$ in~$W_t^{C_T}$. It then follows that the staged subtrees $\mathcal{ST}(s_a)$ and $\mathcal{ST}(s_b)$ of~$\mathcal{ST}_{\!\!\infty}$ must extend until the end of time-slice~$T$. By the definition of a TOG, this actually implies that  $\mathcal{ST}(s_a)$ and $\mathcal{ST}(s_b)$ must be infinite. 

Let $l_t(\lambda)$ be the first situation that happens at time-slice~$t$ in a path~$\lambda$.
Because of time-homogeneity and graphical periodicity conjugate with Lemma~\ref{lem:C_k} if $s_a$ and $s_b$  descend, respectively, from distinct leaf nodes $l_a$ and $l_b$ of a non-empty $\mathcal{T}_{\!-1}$, there is a probabilistic and graphical isomorphism between the infinite staged trees $\mathcal{ST}_{\!\!\infty}(l_\nu(\lambda_a))$ and $\mathcal{ST}_{\!\!\infty}(l_\nu(\lambda_b))$. Thus the situations $s_a$ and $s_b$ must be in a same position in the $N\text{T-DCEG}$~$\mathbb{C}$ which contradicts our initial hypothesis. So, if two situations at time $t,t=0,\ldots,T\!-\!N\!+\!1$, are in different positions in $\mathbb{C}$ then they must be in different positions in $\mathbb{C}_T$.

Observe now that if two situations at time $t, {t\!=\!-1,\ldots,T\!-\!N\!+\!1}$, are in a same position in $\mathbb{C}$ then they must also be in a same position $\mathbb{C}_T$. This happens because if the colourful trees unfolding from these two situations are isomorphic then their corresponding finite colourful subtrees associated with $\mathbb{C}_T$ are also isomorphic. This completes the proof.

\begin{myLemma}
	\label{lem:C_k} 
	Take an $N$T-DCEG supported by a TOG($\mathcal{T}_{\!-1},\mathcal{T}$), such that $\mathcal{T}_{\!-1} \neq \emptyset$, and a CEG~$\mathbb{C}_T$, ${T\in\{2N\!-\!1,2N,\ldots\}}$, in~$\mathcal{F}(\mathbb{C})$. Consider two situations $s_a(t)$ and $s_b(t)$, ${t\in \{-1,0,\ldots,T\!-\!N\}}$,  that descend, respectively, from distinct leaf nodes $l_a$ and $l_b$ of $\mathcal{T}_{\!-1}$. If these two situations are in the same position in $\mathbb{C}_T$, units unfolding from~$l_a$ and~$l_b$ in~$\mathbb{C}$ follow identical time-homogeneous processes at every time-slice~$t$, $t=N\!-\!1,N,\ldots$.
\end{myLemma}

\begin{proof}
	Being in the same position $w_c$ guarantees that there is a graph and probabilistic isomorphism between the staged subtrees $\mathcal{ST}_{\!\!\nu}(s_a)$ and $\mathcal{ST}_{\!\!\nu}(s_b)$, where $\nu=\max(N,t\!+\!N)$. So, we can then map every $s_a$-to-$l_i$ path~$\lambda_a$, $l_i \in l(ST_\nu(s_a))$,  to a $s_b$-to-$l_j$ path~$\lambda_b$, $l_j \in l(ST_\nu(s_b))$. Let $l_t(\lambda)$ be the first situation that happens at time-slice~$t$ in a path~$\lambda$. 
		
Time-homogeneity after time~${N\!-\!1}$ and the graphical periodicity enforced by the TOG($\mathcal{T}_{\!-1},\mathcal{T}$) then guarantees that the finite staged trees $\mathcal{ST}_{\!\!\nu}(l_{\nu-1}(\lambda_a))$ and $\mathcal{ST}_{\!\!\nu}(l_{\nu-1}(\lambda_b))$ are probabilistically and graphically isomorphic. Note that $\nu$ is greater or equal to~$N$. Therefore, if $s_a$ and $s_b$ descend, respectively, from different leaf nodes $l_a$ and $l_b$ of $\mathcal{T}_{\!-1}$, the time-homogeneous processes corresponding to units that unfold from~$l_a$ and~$l_b$ are identical for every time-slice~$t$, $t=N\!-\!1,N,\ldots$.
\end{proof}

%%%%%%%%%%%%%%%%%%%%%%%% 
\section{Proof of Theorem \ref{the:legend}}
\label{app:legend}

\noindent
\emph{
	Take an NT-DCEG $\mathbb{C}=(V,E)$ and define the set of positions $\mathcal{W}_{Head}$ in~$V$, according to Definition~\ref{def:temporal_edges}. In a CEG ${\mathbb{C}_T\!=\!(V_T,\!E_T)}$, ${T \!=\! 2N\!\!-\!{\eta}, 2N\!\!-{\eta\!+\!1},\ldots}$ and~${\mathbb{C}_T \!\in\! \mathcal{F}(\mathbb{C})}$, for every position~$w_a^*(t)$ in the set~$f_t(\mathcal{W}_{Head})$, ${f_t(\mathcal{W}_{Head}) \subseteq V_T}$ and~${t=N,\ldots,T\!-\!N\!+\!{\eta}}$, we have that
	\vspace{-7pt}
	\begin{equation*}
	\Xi_c(w_a^*(t),N\!-\!1)=\Xi_c(w_b^*(N\!-\!1),N\!-\!1),
	\vspace{-7pt}
	\end{equation*}
	where $w_b^*(N\!-\!1)= f_{N\!-\!1}(f_{t}^{-1}(w_a^*(t))) \in V_T$.
	\vspace{5pt}}

Suppose that this result does not hold. Then, according to Theorem~\ref{theo:C_k} a CEG~$\mathbb{C}_T$ contains at least one position~$w_a^*(t)$ in~$f_t(\mathcal{W}_{Head})$, $t=N, \ldots, T\!-\!N\!+\!{\eta}$, such that
${\Xi_c(\!w_a^*(t),\!N\!-\!1) - \Xi_c(\!w_b^*(\!N\!-\!1),\!N\!-\!1)=\Xi_\delta=\{\xi_i(w_a^*(t),\!N\!-\!1)\} \!\neq\! \emptyset}$ 
and $w_b^*(N\!-\!1)\!=\! f_{N\!-\!1}(f_{t}^{-1}(w_a^*(t)))$. Note that in the $N$T-DCEG~$\mathbb{C}$ the positions ${w_a=f_{t}^{-1}(w_a^*(t))}$ and ${w_b=f_{N\!-\!1}^{-1}(w_b^*(N\!-\!1))}$ are the same: $w_a \equiv w_b$.
For every sequence of events~$\xi_i$ in~$\Xi_\delta$ there is therefore a position~$w_c^*(N\!-\!1)$ in~$V_T$, such that ${\xi_i \in \Xi_c(w_c^*(N\!-\!1),N\!-\!1)}$. So, in the infinite event tree associated with $\mathbb{C}$ there are at least two situations $s_a$ and~$s_c$, such that
${s_a \in w_a}$, ${s_c \!\in\! w_c=f_{N\!-\!1}^{-1}(w_c^*(N\!-\!1))}$, and both situations descend from the same leaf node of $\mathcal{T}_{-1}$ if $\mathcal{T}_{-1} \neq \emptyset$. Note that $w_c \in V$. Because of the periodicity and time-homogeneity after time $N\!-\!1$, it follows that there is an isomorphism between the staged subtrees that unfold from situations $s_a$ and $s_c$. Therefore, both situations are in the same position. Thus we have that $w_a \equiv w_c$ in $\mathbb{C}$. This implies that $w_b \equiv w_c$ and so $w_b^*(N\!-\!1) \equiv w_c^*(N\!-\!1)$. This means that $\{\xi_i(w(t),N\!-\!1)\}$ is empty. The result then follows by contradiction.

%%%%%%%%%%%%%%%%%%%%%%%%%%%%%%%%%%%%%%%%%%%%%%%%%%%%%%%%%%%%%%%
\section*{Acknowledgements}
\label{sec:Acknowledgement}
Rodrigo A. Collazo was supported by the Brazilian Navy and CNPq-Brazil [grant number 229058/2013-2]. Jim Q. Smith was supported by the Alan Turing Institute and funded by EPSRC [grant number EP/K039628/1].

%\singlespacing
\fontsize{11pt}{12pt}\selectfont 
%%%%%%%%%%%%%%%%%%%%%%%%%%%%%%%%%%%%%%%%%%%%%%%%%%%%%%%%%%%%%%%
%\section*{References}
\bibliographystyle{elsarticle-num} 
\bibliography{Bibliography_v1}

%% file: NTDCEG_Algorithm.tex
\begin{algorithm}[htbp]
  \DontPrintSemicolon
  
  \KwIn{Finite Event Trees~$\mathcal{T}_{-1}$ and~$\mathcal{T}$}
  
  \KwOut{$N$T-DCEG model}
  
  $\Delta(\mathcal{T}_{N\!-\!1}) \gets \text{TOG}(\mathcal{T}_{-1},\mathcal{T}_{N\!-\!1},N\!-\!1)$.\;
  
  Colour $\Delta(\mathcal{T}_{N\!-\!1})$ to obtain~$\Delta(\mathcal{ST}_{\!\!\!N\!-\!1})$.\;
  
  Construct the tree object~$\Delta(\mathcal{ST}_{\!\!\!2N\!-\eta-1})$ from~$\Delta(\mathcal{ST}_{\!\!\!N\!-\!1})$ using recursively equation~\ref{eq:recursive_structure_tree} based on $\Gamma_t=\Gamma_*$, $\Upsilon(\mathcal{ST}_t)$ and $h_t$, $t=N\!-\!1,\ldots,2N\!-\eta\!-\!2$, defined below.\;  
  
  Identify the position structure~$W$ corresponding to~$\Delta(\mathcal{ST}_{\!\!\!N\!-\!1})$.\;
  
  Obtain the graph~$\mathbb{C}_*$ (Definition~\ref{def:star_graph}). \; 
  	
  Obtain the $N$T-DCEG~$\mathbb{C}$ by merging the nodes of~$\mathbb{C}_*$ according to~$W$.\;
  
  Elicit the conditional probabilities $P(\mathbb{C})$ of $\mathbb{C}$.\;
  
  \Return{$\mathbb{C}$ and $P(\mathbb{C})$}\;
  
\caption{$N$T-DCEG Algorithm \label{alg:NTDCEG_algorithm}}
\end{algorithm}              